\documentclass[12pt]{amsart}
\usepackage{amsfonts}
\usepackage{amsmath}
\usepackage{amsthm}
\usepackage{bm}
\usepackage{bbm}
\usepackage{graphicx}
\usepackage{float}
\usepackage{enumerate}
\usepackage{natbib}
\usepackage{pdfsync}

\usepackage{subfigure}

\usepackage{mathtools}

\numberwithin{equation}{section}

\allowdisplaybreaks

\theoremstyle{plain}
\newtheorem{theorem}{Theorem}[section]
\newtheorem{lemma}[theorem]{Lemma}
\newtheorem{proposition}[theorem]{Proposition}
\newtheorem{corollary}[theorem]{Corollary}
\theoremstyle{definition}

\newtheorem{remark}[theorem]{Remark}

\newcommand{\bM}{{\bf M}}

\newcommand{\bx}{{\bf x}}
\newcommand{\bX}{{\bf X}}
\newcommand{\bN}{{\bf N}}
\newcommand{\by}{{\bf y}}

\newcommand{\bZ}{{\bf Z}}
\newcommand{\ba}{{\bf a}}
\newcommand{\boldb}{{\bf b}}

\newcommand{\bu}{{\bf u}}

\newcommand{\m}{{\bf m}}

\newcommand{\bc}{{\bf c}}

\newcommand{\BX}{{\bf X}}

\newcommand{\BW}{{\bf W}}
\newcommand{\BD}{{\bf D}}

\newcommand{\BY}{{\bf Y}}

\newcommand{\BM}{{\bf M}}

\newcommand{\0}{\mathbf{0}}

\newcommand{\one}{\mathbf{1}}

\newcommand{\bbP}{\mathbb{P}}
\newcommand{\bbE}{\mathbb{E}}

\newcommand{\bbS}{\mathbb{S}}
\newcommand{\bbr}{\mathbb{R}}

\newcommand{\bbs}{\mathbb{S}}

\newcommand{\bma}{\begin{matrix*}[r]}
\newcommand{\ema}{\end{matrix*}}

\newcommand{\cip}{\stackrel{P}{\to}}

 \title{Spectral Learning of Multivariate Extremes}
 
 \author{Marco Avella Medina}
 \address{Department of Statistics \\
 	Columbia University}
 \email{marco.avella@columbia.edu}
 
 \author{Richard A. Davis}
 \address{Department of Statistics \\
 	Columbia University}
 \email{rdavis@stat.columbia.edu}
 
 \author{Gennady Samorodnitsky}
 \address{School of Operations Research and Information Engineering\\
 	Cornell University}
 \email{gs18@cornell.edu}
 
 \numberwithin{equation}{section}
 \thanks{This research was partially
 	supported by  NSF grants DMS-2015379 (Avella Medina and Davis) at Columbia and DMS-2015242 (Samorodnitsky) at Cornell.}

 \keywords{Angular measure, heavy tails, Laplacian, nearest neighbor graphs, regular variation, spectral clustering}%

\begin{document}
	
\begin{abstract}
We propose a spectral clustering algorithm for analyzing the dependence structure of multivariate extremes. More specifically, we focus on the asymptotic  dependence of multivariate extremes characterized by the angular or spectral measure  in extreme value theory. Our work studies the theoretical performance of spectral clustering based on a random $k$-nearest neighbor graph constructed from an extremal sample, i.e., the angular part of random vectors for which the radius exceeds a large threshold.  In particular, we derive the asymptotic distribution of extremes arising from a linear factor model and prove that,  under certain conditions, spectral clustering can consistently identify the clusters of extremes arising in this model. Leveraging this result we propose a simple consistent estimation strategy for learning the angular measure. Our theoretical findings are complemented with numerical experiments illustrating the finite sample performance of our methods.

\end{abstract}\maketitle
	
\section{Introduction} \label {sec:intro}

Multivariate extremes arise when one or more of rare extreme events occur simultaneously. They are of paramount importance for understanding environmental risks such as fires or droughts since they are driven by joint extremes of a number of meteorological variables. Similarly, catastrophic financial events are also of a multivariate nature in financial systems driven by core institutions that are connected. In the above examples one is precisely interested in modeling the dependence between rare individual extremes. Multivariate extreme value theory is an active research area that provides tools for  modeling such events. 

The dependence structure between extreme observations can be complex and typically characterized by different notions of dependence from the ones arising in the non-extreme world. For this reason recent work has sought to rethink various notions of sparsity for extremes \citep{goixetal2017,meyerandwintenberger2019,simpsonetal2020}, concentration inequalities \citep{goixetal2015,clemencconetal2021}, conditional independence \citep{belkinandniyogi2003} and unsupervised learning \citep{chautru2015,cooleyandthibaud2019,janssenandwan2020,dreesandsabourin2021}. See also \cite{engelkeandivanovs2021} for a review of recent developments in the literature of multivariate extremes. Much of this line of research tries to connect important ideas from modern statistics and machine learning to the context of multivariate extremes.  Our work falls in this category as we propose spectral clustering as a tool for learning the dependence structure of multivariate extremes.

Spectral clustering \citep{vonluxburg2007} and related techniques are very popular and have found success in various applications such as parallel computing \citep{hendrickson1995improved,vandriesscheandroose1995}, image segmentation \citep{shi2000normalized} and community detection \citep{roheetal2011,leiandrinaldo2015,zhouandamini2019}. The central idea of spectral clustering is to use the eigenvectors of the graph Laplacian matrix constructed from an affinity graph between sample points in order to find clusters in the data. Typically these are obtained by a $K$-means algorithm that take these  graph Laplacian eigenvectors as input. We follow this same principle but use as input to our algorithm  the angular parts of the observations whose norms exceed a certain large threshold i.e.,   a standard spectral clustering algorithm is applied to the graph built over the angular parts of these extreme observations. 

Because of the nature of the extreme events that we study, we leverage tools from multivariate extreme value theory for analyzing the theoretical properties of our spectral clustering algorithm.  In particular, we use multivariate regular variation as a modeling tool since it is closely connected to asymptotic characterizations of multivariate extreme value distributions \citep{resnick2007, resnick2008}. While a precise definition of regular variation is provided in Section \ref{sec:rv}, the basic idea is that a $d$-dimensional random vector $\bX$ is regularly varying if  the distribution of the angular part  $\bX/\|\bX\|$  stabilizes (i.e., converges in distribution) as the radial part $\|\bX\|$ becomes large and that the radial part has Pareto-like tails.  The dependence structure is then governed by the asymptotic distribution of the  limiting angular part.  In this paper, we consider clustering of the angular parts, which live on a $d$-dimensional unit sphere, of   observations with large radii.  Learning this measure is challenging because of its multivariate nature and because only a small fraction of the data is considered to be extremes, i.e., those observations whose radii are sufficiently large, are retained for estimation. In contrast, standard modeling approaches built on  parametric models are hard to extend to larger dimensions because of  their lack of flexibility and computational complexity \cite{davisonandhuser2015}.

We will explore the use of spectral clustering for learning the angular measure. The performance of the algorithm  critically depends on the properties of the random graph that it takes as input. We will focus on $k$-nearest neighbor graphs and hence a decision has to be made about the size of $k$ for  constructing the random graph. In this work we study this question by focusing on a linear factor model. We characterize the asymptotic distribution of the multivariate extremes generated from this model and show that their dependence structure is captured by a discrete angular measure in the limit. We establish a rate of convergence for the angular components of the extremes to their discrete limits.  This is a key step in deriving a theoretically valid range of   numbers of $k$-nearest neighbors for constructing a nearest  neighbor graph that one should consider in order to guarantee that spectral clustering can be successfully used to learn the asymptotic angular measure.

From a methodological perspective, the work of \cite{janssenandwan2020} is perhaps the closest to  our approach since they also provide a clustering algorithm for extremes. Their method is however very different as it  is based on spherical $k$-means \citep{dhillon&modha2001}, a variant of $k$-means that replaces the usual square loss minimization by an angular dissimilarity measure minimization.  The data-generating model we consider is a natural factor model  that can be viewed as a generalization of the max-linear model considered in \cite{janssenandwan2020}. We 
characterize the limiting distribution of the extremes in this model. We rigorously study the extremal nearest neighbor graphs and show that their connected components can identify the clusters of extremes of our factor model.  By construction our algorithm is computationally tractable and model agnostic, so it has a  potential of working well beyond the setting covered by our theory.

The rest of the paper is organized as follows. Section 2 provides some background notions from  multivariate regular variation necessary for our analysis. Section 3 introduces the proposed spectral clustering algorithm for extremes. In Section 4 we introduce our linear factor model (LFM) and derive the asymptotic distribution of the angular components $\bX/\|\bX\|$ of observations with high threshold exceedances i.e., observations $\bX$ with very large $\|\bX\|$. In Section 5 we study the behavior of $k$-nearest neighbor graphs constructed using a sample of extremes. Section 6 contains a number of numerical examples that illustrate our proposed method.  We show  in Section \ref{sec:lfm} that for a large range of values of $k$ the connected components of the nearest neighbor graph consistently identify the clusters of extremes arising from the linear factor model. This includes an examination of LFM with added noise.  The spectral clustering method is still able to estimate the signal reasonably well.  The good numerical performance of the method in the LFM plus noise context suggests that it might work well in more general settings.  The spectral clustering method is also applied to an environmental data set consisting of daily measurements of five air pollutants over both winter and summer seasons.  The analysis suggests that in modeling the extremes, a LFM model with 5 clusters seems appropriate. 
Moreover, viewed as a time series, the extremal dependence for O3 and NO2 does not extend beyond a second-day time lag.  Proofs of the technical results in the body of the paper  and their complements are contained in the appendix.

\section{Background on multivariate regular variation} \label{sec:rv}

Regular variation  is often the starting point in modeling heavy-tailed data. We will make regular use of this assumption throughout this work.  A random vector $\bX=(X_1,\ldots,X_d)^\top$ is said to be regularly varying with exponent $\alpha>0$ if for some norm $\|\cdot\|$ on $\bbr^d$ and some probability measure $\Gamma$ on the unit
sphere $\bbs^{d-1}$ in $\bbr^d$, the following  limits hold:

\begin{eqnarray}\label{eq:mrv}
\lim_{r\to\infty} \bbP\bigl( \BX/\|\BX\|\in \cdot\,\mid \, \|\BX\|>r\bigr)\Rightarrow \Gamma(\cdot)
\end{eqnarray}
and
\begin{eqnarray}\label{eq:mrv2}
\lim_{r\to\infty}\frac{\bbP\bigl( \|\BX\|>rx\bigr)}{\bbP\bigl( \|\BX\|>r\bigr)}
=x^{-\alpha}
\end{eqnarray}
for all $x>0$, where $\Rightarrow$ denotes weak convergence on $\bbs^{d-1}$.  
 In other words,  the law of the angular component $\BX/\|\bX\|$  stabilizes  as the radial component becomes large, and the radial component is regularly varying (equation \eqref{eq:mrv2}) with index $\alpha$. The limit probability measure $\Gamma$ is called the {\it angular measure} (or spectral measure) and describes how likely the
extremal observations are to point in different
directions. In other words, the angular measure describes the limiting extremal angle for high threshold exceedances that correspond to large $\|\bX\|$.  The support of this measure is particularly
important since it shows which directions of the extremes are feasible
and which are not feasible. Throughout the rest of paper we will take $\|\cdot\|$ to be the Euclidean norm.

For example, if $\bX$ has a spherically symmetric distribution and the radius $\|\bX\|$ has a Pareto distribution with index $\alpha$, then $\bX$ is regularly varying with angular measure that is uniform on $\bbs^{d-1}$.  In this case, the random vector is equally likely to have extremes in  any direction  so we do not expect extremes to be clustered.  On the other hand, consider observations generated from a univariate MA(3) process given by $Y_t=Z_t+.5Z_{t-1}-.6 Z_{t-2}+1.5Z_{t-3}$, where $\{Z_t\}$ is an iid sequence of  symmetric stable random variables with index $\alpha=1.8$. The bivariate vector $\bX_t=(Y_t,Y_{t-1})^\top$ is regularly  varying and the scatter plot of $Y_{t}$ vs $Y_{t-1}$ is displayed in the left panel of Figure \ref{fig:MA3_data}.  Notice that for large values of $\|\bX_t\|$, the points align themselves on rays.  In the right panel is a plot of $\bX_t/\|\bX_t\|$ for those values of $\|\bX_t\|$ that exceed the 99.8\% empirical quantile of the radii and are grouped in 10 clusters. In this particular case, the spectral distribution consists of 10 point masses (5 pairs of symmetric point masses, indicated by arrows emanating from the origin).  
\begin{figure} [H]
\centering	
		\includegraphics[scale=0.5]{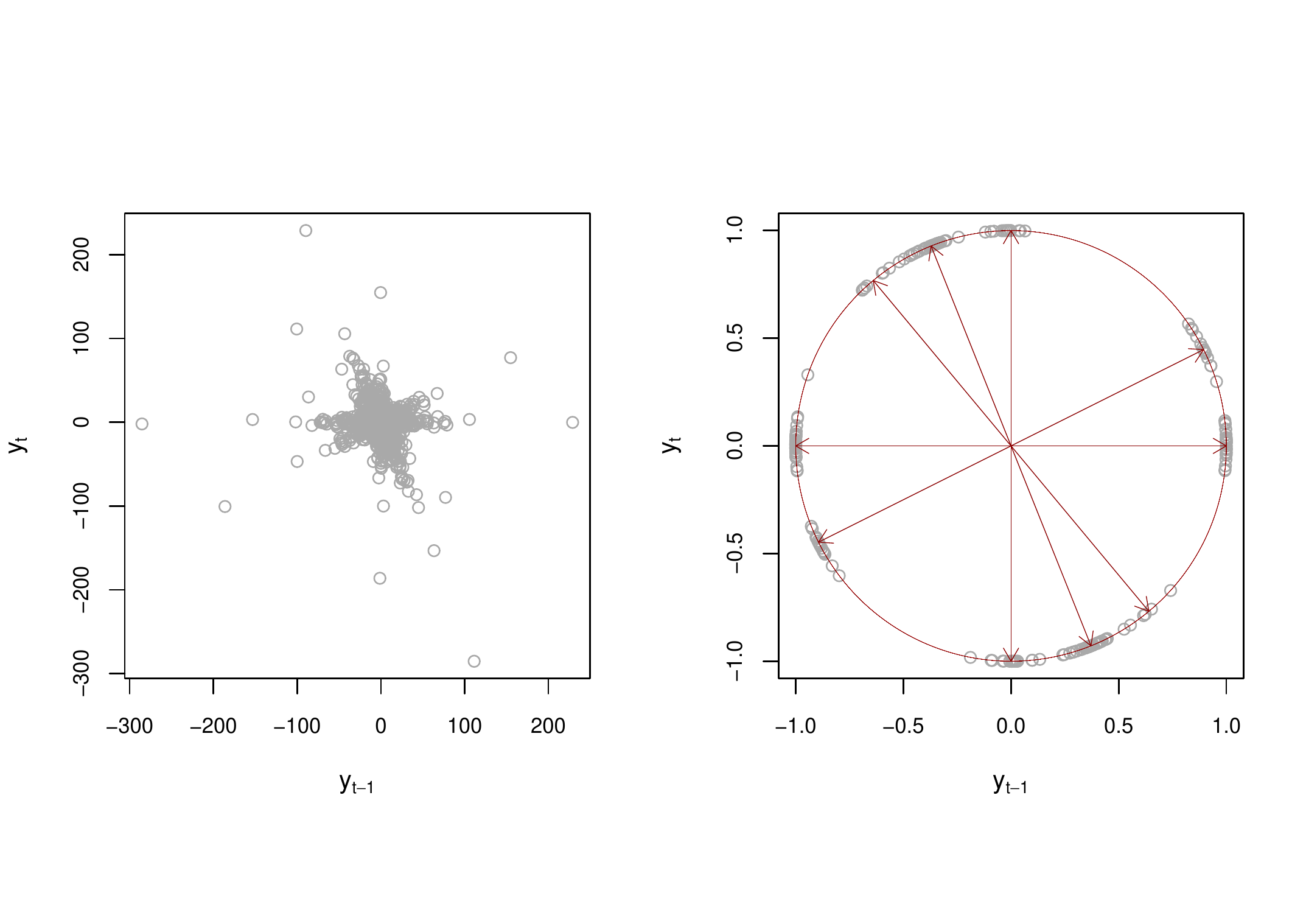}
	\caption{Scatter plot of $(Y_t,Y_{t-1})$ for an MA(3) process (left); spectral measure on $\bbs^1$ (right)}
	\label{fig:MA3_data}
\end{figure}
This simple example illustrates the challenge in finding meaningful low dimensional regions supporting the extremes.  
In a series of papers (see \cite{meyerandwintenberger2019}, \cite{dreesandsabourin2021}, \cite{cooleyandthibaud2019}), PCA-like analyses have been applied to finding low-dimensional subspaces that  contain the bulk of the support of the spectral measure.  As seen in this MA(3) example, such strategies might not be well suited  for extracting the key features in the extremes which do not necessarily live neatly
on  a small number of subspaces.  The approach taken here will be to use spectral clustering to learn the angular measure.  While machine learning ideas provide guidance for thinking about and  addressing multivariate extremes, the very nature of rare events will require us to borrow ideas from the theory of multivariate regular variation to analyze the extremal nearest neighbor graphs used by our algorithm.

\section{Spectral clustering}

In this section we describe how to construct random graphs based on a sample of extremes and how to use such graphs to find clusters of  extremes via a simple spectral clustering algorithm.

\subsection{Constructing random graphs} 

Starting with a sample of $d$-dimensional observations
$\BX_i, \, i=1,\ldots, n$, one first needs to 
identify the extremal part of the sample, on which the extremal
estimation will be performed. This is often done by selecting a
high threshold $u_n$ and assigning to the extremal part of the sample the 
observations $\BX_i$ satisfying $\|\BX_i\|>u_n$.

Assume that
$N_n$ observations $\BX_{i}$ (with $i$ in some set $\mathcal{V}_n$ of
cardinality $N_n$) are in the extremal part of the sample.  Associated with each $i\in \mathcal{V}_n$, is  the angular component  of the observation $\BX_i/\|\BX_i\|$ 
that lives on the unit sphere.  This allows us to think of the points in $\mathcal{V}_n$ as points on the unit sphere, forming nodes in a simple graph.  We connect
nodes $i_1$ and $i_2$ by an edge according to  a certain rule. One possible rule chooses  $\epsilon>0$ and connects  $i_1,i_2\in \mathcal{V}_n$ by 
an edge if 
\begin{align}\label{e:rule.1}
\rho\bigl( \BX_{i_1}/\|\BX_{i_1}\|,
\BX_{i_2}/\|\BX_{i_2}\|\bigr)\leq\epsilon\,.
\end{align}
One often uses the usual Euclidean distance $\rho$ on the unit sphere $\bbs^{d-1}$
in \eqref{e:rule.1}; but a distance function on the unit sphere could also be used.
The random set of edges  $\mathcal{E}_n$ created in this fashion define an $\epsilon$-neighborhood graph. 

In what follows, we will focus on  a different rule, leading to the \textit{$k$-Nearest Neighbor graphs}  ($k$-NN graphs).  This rule asserts that a  node $i_1\in \mathcal{V}_n$ is connected to a node $i_2\in \mathcal{V}_n$ if the point on the unit sphere corresponding to $i_2$ is among the $k$-nearest neighbors of  the point corresponding to $i_1$, according to some distance function. This definition leads to a directed graph because the neighborhood relationship is not symmetric. There are two natural ways of making this graph undirected. The first one is to  connect $i_1$ and $i_2$ with an undirected edge if either $i_1$ is among the $k$-nearest neighbors of $i_2$ or $i_2$ is among the $k$-nearest neighbors of $i_1$. The second one connects $i_1$ and $i_2$ only if both conditions are met, i.e., when $i_1$ and $i_2$ are mutual nearest neighbors  (hence the resulting graph is usually called \textit{mutual} $k$-nearest neighbor graph).  Our main results apply to both constructions. We work with weighted graphs,  where we assign to the edges a weight equal to the distance between the points on the unit sphere defining the nodes. More specifically, we will take as input to our algorithm the \textit{weighted adjacency matrix} $\BW=[w_{i_1i_2}]_{i_1,i_2\in \mathcal{V}_n}$ and 
\begin{equation}\label{eq:W_matrix}
 w_{i_1i_2}=\begin{cases} d\bigl( \BX_{i_1}/\|\BX_{i_1},\|,
\BX_{i_2}/\|\BX_{i_2}\|\bigr) & \mbox{if $i_1$ and $i_2$ are connected,} \\0, & \mbox{if $i_1$ and $i_2$ are not connected.}    
 \end{cases}
\end{equation}
When defining the weights in \eqref{eq:W_matrix} $d$ is a certain kernel; a typical example of such a kernel e.g., $d(\bx,\by)=\exp(-\|\bx-\by\|)$, is used in the examples of Section \ref{sec:numerics}. 
In the following subsections, we  describe our algorithm  and highlight the theoretical challenges.

\subsection{The algorithm}

The degree of a node $i\in \mathcal{V}_n$ is defined as  
\begin{equation*}
d_i=\sum_{j\in \mathcal{V}_n}w_{ij}.    
\end{equation*}
The degree matrix $D$ is defined as the diagonal matrix with diagonal elements $[d_i]_{i\in \mathcal{V}_n}$ and the normalized symmetric graph Laplacian matrix is defined as 
\begin{equation}\label{eq:Lap}
    L=I-D^{-1/2}WD^{-1/2},
\end{equation}
where $I$ is the identity matrix. The spectral clustering algorithm of \cite{ngetal2002} proceeds as follows:
\begin{enumerate}
\item
Compute the first $m$ eigenvectors $\bu_1,\dots,\bu_m$ of $L$ (i.e., the eigenvectors corresponding to the $m$ smallest eigenvalues of $L$)
and define an $N_n\times m$ matrix $U$ using these eigenvectors.
\item
Form an $N_n\times m$ matrix $V$ by normalizing the rows of $U$ to have unit norm.
\item
Treating each of the $N_n$ rows of $V$ as a vector in $\mathbb{R}^m$, cluster them into $m$ clusters $C_1,\dots,C_m$ using the $K$-means algorithm.
\item
 Assign the original points $\bX_{i}$  to cluster $C_j$ if and only if row $i$ of the matrix $V$ was assigned to cluster $C_j$.
\end{enumerate}
The motivation for this algorithm is described below.  %

\subsection{Connected components, Laplacian and $k$-nearest neighbor graph.}

 We say that a subset $\mathcal{A}\subset\mathcal{V}_n$ of the vertices of a graph is connected if any two vertices in $\mathcal{A}$  can be joined by a path of edges such that all intermediate vertices also lie in $\mathcal{A}$. If $\mathcal{A}$ is connected and there are no connections between $\mathcal{A}$ and $\mathcal{V}_n\setminus\mathcal{A}$, then $\mathcal{A}$ is called a connected component. It is well known that the number of connected components of a graph $G$ is related to the spectrum of its symmetric graph Laplacian. This is formalized in the following proposition \citep[Proposition 2]{vonluxburg2007}.

\begin{proposition}
\label{prop:Lap_spectrum}
  Let $\mathcal{G}=(\mathcal{V},\mathcal{E})$ be an undirected graph with non-negative weights. Then the multiplicity $m$ of the eigenvalue $0$ of $L$ equals the number of connected  components   $\mathcal{A}_1,\dots,\mathcal{A}_m$ in the graph. The eigenspace of eigenvalue $0$ is spanned by the indicator vectors $\delta_{\mathcal{A}_1},\dots,\delta_{\mathcal{A}_m}$ of those components. 
 \end{proposition}

It follows from this result that if   the spectral clustering algorithm is applied to a graph with the  number of connected components equal to the parameter $m$ in the algorithm, then 
the algorithm will identify the connected components. In Sections \ref{sec:spectrum} and \ref{sec:d2} we derive the asymptotic angular distribution of multivariate extremes arising from a linear factor model and provide the relevant asymptotic theory for the connected components of the $k_n$-nearest neighbor graph constructed using the angular components of the extremes.  Specifically, it will be rigorously  established that under certain conditions the spectral clustering of the resulting graph consistently estimates  the support of the spectral measure of multivariate extremes arising from this model.

\section{ Linear factor model and convergence of the angular components}\label{sec:spectrum}

  \bigskip

We now introduce the generative model that we will be studying in this paper. Let $\bX$ be a $d$-dimensional random vector defined by  the following linear factor model (LFM)
\begin{equation} \label{e:X}
 \bX=A\bZ\,, 
\end{equation}
where $A=[a_{ij}]_{i=1,\ldots d;j=1,\ldots p}$ is a $d\times p$ matrix of nonnegative elements and $\bZ$ is a $p$-dimensional random vector  of factors consisting of independent and identically distributed random variables, that are either nonnegative or symmetric, and 
have asymptotically Pareto tails, i.e.,
\begin{equation} \label{e:regvar.Z}
 \bbP(Z_1>z)\sim cz^{-\alpha},\ \mbox{as $z\to\infty$}
\end{equation}
for some $\alpha>0$ and $c>0$. Note that we write $f(x)\sim g(x)$ as $x\to\infty$ to mean that $\lim_{x\to\infty}f(x)/g(x)=1$. In the examples section, we will add noise to the model in \eqref{e:X} in which case, the model corresponds to a {\it standard} heavy-tailed linear factor model. 
 One can think of \eqref{e:X} as a {linear} version of the max-linear model studied in \cite{janssenandwan2020}, which has the same spectral distribution. We will also relax the assumption that the matrix $A$ is non-negative and will allow the noise to be symmetric.  In this case, the spectral distribution of the model is no longer constrained to the positive quadrant of $\bbs^{d-1}$. Related max-linear models have also been considered in the context of time series models for extremes \citep{davisandresnick1989,halletal2002} and more recently in the context of structural equation models \citep{gissiblandkluppelberg2018,kluppelbergandlauritzen2019}.   The asymptotic Pareto assumption in \eqref{e:regvar.Z} can be weakened to   regular variation, at least for the main results in this section.  However, this comes at the expense of assuming a more intrinsically complex set of conditions on the choice of  thresholding sequences.  Additional conditions, such as existence and properties of the density function of the noise, are required for the proofs of the results in Section \ref{sec:d2}.    

It follows immediately from \eqref{e:X} and \eqref{e:regvar.Z} (see, for example, \cite{basrak2002characterization}, Proposition A.1) that
$\BX$ is a multivariate regularly varying
random vector satisfying  \eqref{eq:mrv}; namely, 
\begin{equation} \label{e:multv.regvar}
\lim_{x\to\infty} \bbP\left( \ \frac{\BX}{\|\BX\|}\in\cdot~|~\|\BX\|>x,\right) 
\Rightarrow  \Gamma(\cdot)\,,
\end{equation} 
where $\Rightarrow$ denotes weak convergence on the unit sphere $\bbS^{d-1}$, 
 $\Gamma$ is a discrete probability measure on $\bbS^{d-1}$ that, in the nonnegative case, puts mass $\|\ba^{(k)}\|^\alpha/w$ at   $\ba^{(k)}/\|\ba^{(k)}\|$ for $k=1,\ldots,p$, where $\ba^{(k)}=(a_{1k},a_{2k},\ldots,a_{dk})^\top\,,
$
%
%
is the $k^{th}$ column of the matrix $A$ and 
\begin{equation} \label{e:total.w}
w=\sum_{k=1}^p \|\ba^{(k)}\|^\alpha\,.
\end{equation}
In other words, $\Gamma$ has the representation
\begin{equation} \label{e:spectral.m}
  \Gamma(\cdot) =w^{-1}\sum_{k=1}^p \|\ba^{(k)}\|^\alpha
       \delta_{\frac{\ba^{(k)}}{\|\ba^{(k)}\|}}(\cdot)\,,
\end{equation}
where $\delta_x(\cdot)$ is the Dirac measure that puts unit mass at $x$. On the other hand, in the symmetric case, $\Gamma$ puts mass $\|\ba^{(k)}\|^\alpha/2w$ at   $\pm\ba^{(k)}/\|\ba^{(k)}\|$ for $k=1,\ldots,p$. That is, the number of point masses in the symmetric case is double of that number in the nonnegative case. \footnote{Without much additional effort, one could consider the case that the tails of $Z_1$ are balanced in the sense that $\lim_{x\to\infty}\bbP(Z_1>x)/\bbP(|Z_1|>x)\to p^+\in [0,1].$  The location of the point masses for $\Gamma$ would be exactly the same as in the symmetric case, but with mass $p^\pm\|\ba^{(k)}\|^\alpha/w$ at $\bc_{k,\pm}$ defined in \eqref{eq:ck},  where $p^-=1-p^+$.}

Based on a random sample of iid copies of $\BX_1,\ldots,
\BX_n$ of $\BX$ as above, we construct an estimate of the  location of the point masses that comprise $\Gamma$, i.e., 
\begin{equation}\label{eq:ck1}
\bc_k=\frac{\ba^{(k)}}{\|\ba^{(k)}\|}\,,~k=1,\ldots,p
\end{equation}
in the nonnegative case, and
\begin{equation}\label{eq:ck}
\bc_{k,\pm}=\frac{\pm\ba^{(k)}}{\|\ba^{(k)}\|}\,,~k=1,\ldots,p
\end{equation}
in the symmetric case. 
   Note that these $\bc_k$ are not necessarily distinct.  Intuitively,  for large $n$, the {\it angular parts} $\BX_i/\|\BX_i\|$ of the sample  for which $\|\BX_i\|$ is large, will cluster around these $\bc_k$.  In fact, we formalize this intuition and provide a {\it rate of convergence} for the limiting extremal angles with high threshold exceedances in the next theorem. This will be a key ingredient in our convergence analysis of extremal $k$-NN graphs. For the ease of notation we will prove the following result in the nonnegative case; the symmetric case follows by simply doubling the number of points on the sphere.

\begin{theorem}\label{l:weak.limit}
If $(u_n)$ is a sequence converging to infinity
as $n\to\infty$, then, in the nonnegative case, for any $j=1,\ldots, p$, the conditional law of
	$$
	u_n\bigl( \bX/\|\bX\|-\bc_j\bigr) 
	$$
	given $\|\bX\|>u_n,\, Z_j>u_n/w^{1/\alpha}$, ($w$ defined in \eqref{e:total.w}) converges weakly to the law of 
	$$
	\frac{1}{\|\ba^{(j)}\|^{2}W_\alpha}
	\left(S^*_{1,-j},  \ldots, S^*_{d,-j}\right)^\top,\,
	$$
	where $W_\alpha$ has a Pareto distribution (i.e. $\bbP(W_\alpha>x)=x^{-\alpha},\,x\ge1$) that is independent of $Z_1,\ldots,Z_p$,
	\begin{equation}\label{eq:Lstar}
	S^*_{l,-j}=\sum_{i=1}^d\left(a_{ij}^2X_{l,-j}-a_{lj}a_{ij}X_{i,-j}\right)\,,
	\end{equation}
	and
	\begin{equation}\label{eq:L}
	X_{i,-j}=X_i-a_{ij}Z_j=\sum_{\substack{m=1,\\m\ne j}}^pa_{im}Z_m\,.
	\end{equation}

\end{theorem}
\begin{proof}
	We start by observing that the conditional law of 
	\begin{equation}\label{eq:zp}
	\bigl( Z_1,\ldots, Z_{j-1},Z_j/u_n,Z_{j+1},\ldots,
	Z_{p}\bigr)^\top
	\end{equation}
	given $\|\bX\|^2>u_n^2$, $Z_{j}>u_n/w^{1/\alpha}$, converges
	in distribution, as $n\to\infty$, to the law of
	\begin{equation}\label{eq:zp2}
	\bigl( Z_1,\ldots, Z_{j-1},W_\alpha/w_j,Z_{j+1},\ldots,
	Z_{p}\bigr)^\top,
	\end{equation}
	where $w_j=\|\ba^{(j)}\|$ 
 The main ingredients in establishing this  result is to note that $Z_j^2$ is regularly varying with index $\alpha/2$ while for $i\ne j$, $Z_iZ_j$ is regularly varying with index $\alpha$ (see \cite{embrechts1980closure}; Theorem 3).  {Moreover, from the convolution closure property for sums of independent regularly varying random variables, it follows easily that
\begin{eqnarray*}
\bbP(Z_j>u_nx\,,\sum_{k=1}^d(\sum_{i=1}^pa_{ki}Z_i)^2>u_n^2,\, )&\sim&\bbP(Z_j>u_nx,\sum_{i=1}^pw_i^2Z_i^2>u_n^2)\\
&\sim&\sum_{i=1}^p\bbP(Z_j>u_nx,w_i^2Z_i^2>u_n^2)\\
&\sim&\bbP(Z_j>u_nx)\,.
\end{eqnarray*}} 
 {Now to finish the proof of \eqref{eq:zp2}, we use these relations and note that for  $x\ge 1/w_j$ (and hence $x\ge 1/w^{1/\alpha}$),}   
\begin{eqnarray*}
\bbP(Z_j>u_nx\,|\,\|\bX\|^2>u_n^2,\, Z_j^2>u_n^2/w^{2/\alpha})&\sim&\frac{\bbP(Z_j>u_nx,\sum_{i=1}^pw_i^2Z_i^2>u_n^2,Z_j^2>u_n^2/w^{2/\alpha})}{\bbP(\sum_{i=1}^pw_i^2Z_i^2>u_n^2,Z_j^2>u_n^2/w^{2/\alpha})}\\
&\sim&\frac{\bbP(Z_j>u_nx)}{\bbP(Z_j>u_n/w_j)}\\
&\to& w_j^{-\alpha}x^{-\alpha}\\
&=&\bbP(W_\alpha> w_jx)\,.
\end{eqnarray*}
We have
	\begin{align*}
	u_n\bigl( \BX/\|\BX\|-\bc_j\bigr) 
	&= u_n\left( \frac{\left( \sum_{m=1}^p a_{1m}Z_{m}, \ldots, \sum_{m=1}^p
		a_{dm}Z_{m}\right)^\top}{\|\bX\|}
	-\frac{\ba^{(j)}}{\|\ba^{(j)}\|}\right)\\
	&=(V_1,\ldots, V_d)^\top, 
	\end{align*}
	say. For $l=1,\ldots, d$, 
	\begin{align} \notag
	V_l &= u_n\frac{w_j\sum_{m=1}^pa_{lm}Z_{m}-a_{lj}
		\|\bX\|}{w_j\|\bX\|}\\
	&= u_n\frac{w_j^2 \left(\sum_{m=1}^pa_{lm}Z_{m}\right)^2-a_{lj}^2
		\|\bX\|^2}{w_j\|\bX\|\left(w_j\sum_{m=1}^pa_{lm}Z_{m}+a_{lj}
		\|\bX\|
		\right)}\,. \label{e:Vl.1}
	\end{align}
	We note that the numerator of \eqref{e:Vl.1} reduces to the following expression where the $Z_j^2$ terms  cancel out
	\begin{eqnarray}
	Num_l:&=&2a_{lj}Z_j \left(\|\ba^{(j)}\|^2X_{l,-j}-a_{lj}\sum_{i=1}^da_{i,j}X_{i,-j}\right) + \|\ba^{(i)}\|^2X_{l,-j}^2-a_{lj}^2 \sum_{i=1}^dX_{i,-j}^2 \nonumber \\
	&=&2a_{lj}Z_j \left(\sum_{i=1}^d\left(a_{ij}^2X_{l,-j}-a_{lj}a_{ij}X_{i,-j}\right)\right) + \|\ba^{(i)}\|^2X_{l,-j}^2-a_{lj}^2 \sum_{i=1}^dX_{i,-j}^2\,,  \label{e:num.1}
	\end{eqnarray}
where  $X_{i,-j}$ is as defined in \eqref{eq:L}.
	The denominator of \eqref{e:Vl.1} is handled in a similar way, but this time the $Z_j^2$ terms do not cancel. Indeed,
	since 
	\begin{align*}
	\|\bX\|&=\sqrt{\sum_{k=1}^d(a_{kj}^2Z_j^2+2a_{kj}Z_jX_{k,-j}+X_{k,-j}^2)} \\  &=\sqrt{\sum_{k=1}^da_{kj}^2Z_j^2}\sqrt{1+\frac{\sum_{k=1}^d(X_{k,-j}^2+2a_{kj}Z_jX_{k,-j})}{\sum_{k=1}^da_{kj}^2Z_j^2}}\\
	    &=w_j|Z_j|\sqrt{1+\frac{\sum_{k=1}^d(X_{k,-j}^2+2a_{kj}Z_jX_{k,-j})}{w_j^2Z_j^2}},
	\end{align*}
	 we can write
	$$
	Den_l = w_j^2|Z_j|(1+o_p(1))\left(w_ja_{lj}Z_j+R+a_{lj}w_j|Z_j|(1+o_p(1))\right),
	$$
	where $R$ is a linear function   in the variables $Z_1,\ldots, Z_p$ in which $Z_j$ does not appear, and $o_p(1)$ goes to zero in probability given $Z_j>u_n/w^{1/\alpha}$. We view 
	$$
	V_l=u_n\frac{Num_l}{Den_l} =
	\frac{Num_l/u_n}{Den_l/u_n^2}
	$$
	as a ratio of two continuous real-valued functions of the random vector in \eqref{eq:zp} (plus a vanishing term in $Dem_l$), 
	so that the random vector $(V_1,\ldots, V_d)^\top$ becomes a $d$-dimensional vector of such ratios. By  the continuous mapping theorem the random vector $(V_1,\ldots, V_d)^\top$ converges weakly to the $d$-dimensional vector of the ratios of the corresponding functions applied to the random vector in \eqref{eq:zp2}. These result in 
	$$
	2a_{lj}(W_\alpha/w_j)S^*_{l,-j} 
	$$
	in the case of $Num_l$ and in 
	$$
	W_\alpha^2\,2a_{lj}w_j \label{eq:dencon}
	$$
	in the case of $Dem_l$. 
	Putting everything together produces the claim.
\end{proof}

As explained above, the following corollary is an immediate consequence of Theorem \ref{l:weak.limit}.
\begin{corollary} \label{c:weaklim.sym}
Let a sequence of levels $(u_n)$ converging to infinity.  
Then, in the symmetric case, in the notation of Theorem \ref{l:weak.limit}, 
for any $j=1,\ldots, p$, the conditional law of
	$$
	u_n\bigl( \bX/\|\bX\|-\pm\bc_j\bigr) 
	$$
	given $\|\bX\|>u_n,\, \pm Z_j>u_n/w^{1/\alpha}$, converges weakly to the law of 
	$$
	\frac{\pm 1}{\|\ba^{(j)}\|^{2}W_\alpha}
	\left(S^*_{1,-j},  \ldots, S^*_{d,-j}\right)^\top.
	$$
\end{corollary}

\begin{remark} In the nonnegative case, 
 Theorem \ref{l:weak.limit} addresses the conditional convergence of $\BX/\|\BX\|$,
 given $\|\bX\|>u_n$, $Z_{j}>u_n/w^{1/\alpha}$,
 to the location $\bc_j$ of the corresponding atom of the spectral measure. It is also possible to address a conditional convergence to the mass $w^{-1}\|\ba^{(j)}\|^{ \alpha}$ of this atom. Indeed,
 \begin{equation} \label{e:mass.j}
 \bbP\bigl( Z_{j}>u_n/w^{1/\alpha}\big| \|\bX\|>u_n\bigr) \to w^{-1}\|\ba^{(j)}\|^\alpha
 \end{equation}
 as $n\to\infty$. To see this, write
 \begin{align*}
  \bbP\bigl( Z_{j}>u_n/w^{1/\alpha}\big| \|\bX\|>u_n\bigr) = \frac{u_n^{\alpha}\bbP\bigl( Z_{j}>u_n/w^{1/\alpha}, \|\bX\|>u_n\bigr)}  {u_n^\alpha \bbP\bigl(\|\bX\|>u_n\bigr) }, 
\end{align*}
 and the numerator converges to $c\|\ba^{(j)}\|^\alpha$, while the denominator converges to $cw$. If one strengthens the asymptotic Pareto tails assumption \eqref{e:regvar.Z} to include the rate of convergence to the limit, then one would able to establish the rate of convergence in \eqref{e:mass.j} as well. The situation is similar in the symmetric case. We do not pursue this in the present paper. 
 \end{remark}

We now explore the connection between large values of the underlying factors $Z_{i1},\dots,Z_{ip}$ and large values of $\|\bX_i\|$. We will see that under certain conditions, high threshold exceedances of $\|\bX_i\|$  are generated by only one underlying factor $Z_{ij}$, $j=1,\dots,p.$ This will be important for our analysis of extremal $k$-NN graphs which  will require  additional
assumptions on the rate of growth of $u_n$. Since $\frac{\alpha+2}{\alpha(\alpha+3)}<\alpha^{-1}$, we can further impose that the sequence $(u_n)$ satisfies the growth conditions
\begin{equation} \label{e:un.special}
  n^{-1/\alpha}u_n\to 0~~\mbox{and}~~n^{-(\alpha+2)/(\alpha(\alpha+3))}u_n\to\infty\,,
  \end{equation}
  as $n\to\infty$.
Also note that we may choose a further sequence $(h_n)$ such
that
\begin{equation} \label{e:hn}
h_n\to\infty, \, h_n=o(u_n), \, h_n=o\bigl(
u_n^{(\alpha+1)/2}n^{-1/2}\bigr), \, n^{-1/\alpha}u_nh_n\to\infty
\end{equation}
as $n\to\infty$. Indeed, the choice
$$
h_n= u_n^{(\alpha-1)/4}n^{(2-\alpha)/(4\alpha)}
$$
works for this purpose.

For $n=1,2,\ldots,$ we define the set of indexes corresponding to extreme observations 
\begin{equation} \label{e:In}
  \mathcal{I}_n=\bigl\{ i=1,\ldots, n:\, \|\BX_i\|>u_n\bigr\},
\end{equation}
and denote its cardinality by  $N_n=$card$(\mathcal{I}_n)$. From \eqref{e:regvar.Z} and {\eqref{e:un.special}}, we see that the mean and variance of $N_n/(nu_n^{-\alpha})$ converge to $cw$ and $0$, respectively and hence
that
\begin{equation} \label{e:In.size}
N_n/(nu_n^{-\alpha})\cip cw,\mbox{~as $n\to\infty$.}
\end{equation}

 The following lemma connects exceedances of $u_n$ by $\|\bX_i\|$ with exceedances of $h_n$ by $Z_{ij},\,j=1,\ldots,p$. 

 \begin{lemma} \label{l:only.one}
   Let $(h_n)$ be a sequence satisfying \eqref{e:hn} and consider the event
   $$
{B}_n=\bigl\{ \text{for any $i\in \mathcal{I}_n$ at most one of $Z_{im}, \,
   m=1,\ldots, p$ exceeds $h_n$}\bigr\}.
 $$
 Then $\mathbb{P}({B}_n)\to 1$ as $n\to\infty$. 
 \end{lemma}
 \begin{proof}
  Note that \eqref{e:X} implies that the $m^{th}$ component of the $i^{th}$ observation is of the form
\begin{align} \label{e:X.Y}
 X_{im}=\sum_{j=1}^p a_{mj}Z_{ij}, \,m=1,\ldots,d;~ i=1,\ldots, n\,,
 \end{align}
 where $Z_{i1},\dots,Z_{ip} $ are iid random variables with asymptotic Pareto tails \eqref{e:regvar.Z}. Denote
   \begin{equation} \label{e:ba}
     a^*=d^{1/2}\max\{ a_{mj},\,{m=1,\ldots,d;\, j=1,\ldots,p}\}
     \end{equation} 
 Since $u_n>h_n$ for $n$ large, we have
 \begin{align*}
&\bbP({B}_n^c)\leq \sum_{i=1}^n \bbP\left( \sum_{k=1}^d\bigl(X_{ik}\bigr)^2
   >u_n^2, \,    Z_{im}>h_n \  \text{for two or more of } \ 
          m=1,\ldots, p\right) \\
  &\leq n\bbP\left(a^*\max_{k=1,\ldots, p}Z_{1k}>u_n, \ Z_{1m}>h_n 
    \ \text{for two or more of } \  m=1,\ldots, p\right) \\
 &\leq \sum_{k=2}^p{p \choose k}n \bbP(Z_1>h_n) \bbP\left(
   Z_1>u_n/a^*\right)\to 0
 \end{align*}
by the last property in \eqref{e:hn} and \eqref{e:regvar.Z}. This proves the lemma.
 
 \end{proof}

Equipped with Lemma \ref{l:only.one} we can now proceed to bound the distance between the observed angular parts of the multivariate extremes and their corresponding theoretical asymptotic atoms. Assume, for a moment, that we are in the nonnegative case. We already know that for large $n$, we have that for every $i\in \mathcal{I}_n$ one of the values of $Z_{im}, \,
m=1,\ldots, p$ must exceed $u_n /a^*$ and
all other values of these variables cannot exceed $h_n$. We now define the sets of indexes corresponding to extremes generated by each of the individual factors i.e. we define for $j=1,\ldots, p$ 
\begin{equation} \label{e:In.j}
  \mathcal{I}_n^{(j)}=\left\{ i=1,\ldots, n:\, \|\BX_{i}\|>u_n, \,
  Z_{ij}>u_n/a^*\right\}. 
\end{equation}
Consequently,
\begin{equation} \label{e:split.extremes}
  \mathcal{I}_n=\bigcup_{j=1}^{p}     \mathcal{I}_n^{(j)}
\end{equation}
and by Lemma \ref{l:only.one} for large $n$  this is a disjoint union with probability tending to one. Let
$N_n^{(j)}$ be the cardinality of $ \mathcal{I}_n^{(j)}$, $j=1,\ldots,
p$. Using the fact that $a^*\ge \|\ba^{(j)}\|$ for $j=1,\ldots,p$, the same argument as in \eqref{e:In.size} shows that $ j=1,\ldots 
p$,
\begin{equation} \label{e:In.size.j}
  N_n^{(j)}/(nu_n^{-\alpha})\cip c\|\ba^{(j)}\|^\alpha, \mbox{~as $n\to\infty$.}
\end{equation}
 We enumerate $\BX_{i}/\|\BX_{i}\|,
\, i\in \mathcal{I}_n$ as $\BY_{i}, \, i=1,\ldots, N_n$, a sample on $\bbS^{d-1}$
of random size $N_n$. 
For each $j=1,\ldots, p$, we enumerate $\BX_{i}/\|\BX_{i}\|, 
\, i\in \mathcal{I}_n^{(j)}$ as $\BY_i^{(j)}, \, i=1,\ldots, N_n^{(j)}$, a sample on $\bbS^{d-1}$
of random size $N_n^{(j)}$. It is straightforward (if a bit tedious)
to check the following result.

\begin{lemma}\label{lem:UB}
   For large $n$, on the event ${B}_n$, for $i=1,\ldots,
N_n^{(j)}$, 
\begin{align}\label{e:near.center}
&\Big\| \BY_i^{(j)} - \bc_j
\Big\|\leq \frac{8(a^*)^2}{\|\ba^{(j)}\|^\alpha}\frac{h_n}{u_n}, \ \ j=1,\ldots, p,
\end{align}
where the $\bc_j$ are as defined in \eqref{eq:ck}.
\end{lemma} The situation in the symmetric case is, of course, completely analogous.  It follows from the definition of $h_n$ in \eqref{e:hn} and \eqref{e:near.center} that  the angular components of the extremes are clustered around the centers $\bc_j$. The results in the next section build on Lemma \ref{lem:UB} and provide sufficient conditions for our extremal spectral clustering algorithm to be consistent. For this, we provide a careful asymptotic analysis of the extremal $k$-NN graph used by the algorithm.

 \section{  Asymptotic analysis of the connected components of the extremal  $k$-NN graph}\label{sec:d2}


Our analysis consists of two main components. The first one is to show that the extremes generated by different factors will belong to different components of the extremal $k$-NN graph as long as the cluster centers corresponding to the underlying factors are different. The second part will be to argue that all the extremes generated by an underlying factor will also belong to the same component of the extremal $k$-NN graph. This second step turns out to be the more technical one in our analysis and we will only  establish this result for $d=2$. Along the way we derive a few intermediate results that we also highlight in order to better explain the key ingredients of our argument.  Going forward, in our proofs,  $c>0$ represents a finite and non-zero constant whose value may change from line-to-line.
In the sequel we will assume, without further comments, that the sequence $(u_n)$ satisfies \eqref{e:un.special}. 
The first step of our program is covered by the following proposition.
\begin{proposition} \label{pr:disconn.outside}
  Suppose that $k_n=o\bigl(nu_n^{-\alpha}\bigr)$ as $n\to\infty$. Then
 there is a sequence $({B}_{n,1})$ of events with $\bbP({B}_{n,1})\to 1$ as
 $n\to\infty$ such that, 
for all $n$ large enough, on the event ${B}_{n,1}$, any two points 
$\BY_{i_1}^{(j_1)}$ and 
$\BY_{i_2}^{(j_2)}$, $ i_1=1,\ldots, N_n^{(j_1)}$, $
i_2=1,\ldots, N_n^{(j_2)}$ will belong to two different
connected components of the $k_n$-{NN} graph if $\bc_{j_1}\not=\bc_{j_2}$.
\end{proposition} 
\begin{proof}
  Define
  \begin{equation} \label{e:An1}
 {B}_{n,1}={B}_n\cap \bigl\{ N_n^{(j)}>k_n\ \ \text{for $j=1,\ldots,
   p$}\bigr\}. 
\end{equation}
By Lemma \ref{l:only.one}, \eqref{e:In.size.j} and the assumption on
$k_n$ , $\bbP({B}_{n,1})\to 1$ as $n\to\infty$. By \eqref{e:near.center}
and the triangle inequality, on the events ${B}_{n,1}$, 
any point $\BY_{i_1}^{(j_1)}$ has at least $k_n$
neighbours of the type $\BY_i^{(j_1)}$, $i=1,\ldots,
N_n^{(j_1)}, \, i\not= i_1$, that are within distance of
$ c\cdot h_n/u_n$ from it. On the other hand, by \eqref{e:near.center}
and the triangle inequality,
its distance from any point $\BY_{i_2}^{(j_2)}$ with
$\bc_{j_1}\not=\bc_{j_2}$ cannot be smaller than
$$
\|\bc_{j_1}-\bc_{j_2}\|-c\cdot h_n/u_n.
$$
Therefore, for large $n$, the latter point cannot be among the
$k_n$-nearest neighbours of $\BY_{i_1}^{(j_1)}$. 
\end{proof}

We now embark on the second step of our program and establish that, at least in the case $d=2$,
under appropriate conditions, the points
$\BY_i^{(j)}$, $i=1,\ldots, N_n^{(j)}$, belong, with
high probability, to the same connected component in the $k_n$-NN graph. 
We start by investigating the deviations of these points from the
center of the cluster, $\bc_j$, defined in \eqref{eq:ck}. Since the points $\BY_i^{(j)}$, $i=1,\ldots, N_n^{(j)}$ are treated as independent, the following result is essentially immediate from Theorem \ref{l:weak.limit}.
\begin{lemma}\label{l:weak.limit2}
  In the nonnegative case, for any $j=1,\ldots, p$, the conditional law of
  $$
  u_n\bigl( \BY_1^{(j)}-\bc_j\bigr) 
  $$
  given $N_n^{(j)}\geq 1$, converges weakly to the law of 
$$
\frac{1}{\|\ba^{(j)}\|^{2}W_\alpha}
\left(S^*_{1,-j},  \ldots, S^*_{d,-j}\right)^T
$$
that is specified in the statement of Theorem \ref{l:weak.limit}. An analogous statement holds in the symmetric case.

\end{lemma}

\begin{remark} \label{rk:lin.dep}  
 It is a straightforward calculation to check that, if $j=1,\ldots,d$, 
 \begin{equation} \label{e:lin.dep1}
   \sum_{l=1}^d a_{lj}S_{l,-j}^*=\sum_{l=1}^d\sum_{i=1}^d\left(a_{lj}a_{ij}^2X_{l,-j}-
   a_{lj}^2a_{ij}X_{i,-j}\right)=0\,.
 \end{equation}
Therefore, the normalized deviations of the points
$\BY_i^{(j)}$, $i=1,\ldots, N_n^{(j)}$ from the center
of the $j$th cluster are, in the limit, supported by a
$(d-1)$-dimensional subspace.  
\end{remark}

Using the information in Lemma \ref{l:weak.limit2} we now proceed to
prove that under appropriate conditions, the points
$\BY_i^{(j)}$, $i=1,\ldots, N_n^{(j)}$, belong, with
high probability, to the same connected component of the $k_n$-NN graph. We will need some additional notation in order to state the result. For a fixed $j=1,\ldots, d$ we write for $m=1,\ldots,d$, 
\begin{align} \label{e:Y}
 Y_{im}^{(j)}=\sum_{l=1}^p a_{ml}Z^{(*,j)}_{il}, \, i=1,\ldots,
  N_n^{(j)}; 
 \end{align}
the notation should be compared with \eqref{e:X.Y}. That is,  $\{(Z^{(*,j)}_{i1},\ldots,Z^{(*,j)}_{ip})^\top, i=1,\ldots, \}$ are iid random vectors distributed according to the conditional distribution of $(Z_1,\ldots,Z_p)^\top$ given $\|\bX\|>u_n,Z_j>u_n/w^{1/\alpha}$.  Since the connectivity
of any nearest neighbor graph is not affected by shifting and
scaling, it is sufficient to consider the $k_n$-NN graph constructed
on the deviations of the points $\BY_i^{(j)}$,
$i=1,\ldots, N_n^{(j)}$ from the cluster center.

 Continuing with the
notation used in the proof of Theorem \ref{l:weak.limit} we isolate the
main term in the deviations from the cluster center by writing 
\begin{align} \label{e:main.term}
  u_n\bigl( \BY_i^{(j)}-\bc_j\bigr)
  &= \left(S_{1,-j}^{(*,j)},  \ldots, S_{d,-j}^{(*,j)}\right)^\top
  /\bigl(w_j^{2}(Z^{(*,j)}_{ji}/u_n)\bigr) \\
&\quad\quad+
  \left[u_n\bigl( \BY_i^{(j)}-\bc_j\bigr) -
  \left(S^{(*,j)}_{1,-j},  \ldots, S^{(*,j)}_{d,-j}\right)^\top/\bigl(w_j^{2}(Z^{(*,j)}_{ji}/u_n)\bigr) \right] \notag \\
 &= \BM^{(i)}+\BD^{(i)}, \quad i=1,\ldots, N_n^{(j)}, \notag 
\end{align}
where $ S_{l,-j}^{(*,j)}	=Y_{il}^{(j)}-a_{lj}Z_{ij}^{(*,j)}=\sum_{\substack{m=1,\\m\ne j}}^pa_{lm}Z_{im}^{(*,j)}\,$ is analogous to   \eqref{eq:L}.
 In the case $d=2$, it follows from
\eqref{e:lin.dep1}  that for some nonzero deterministic vector $\boldb$ in
$\bbr^2$,
$$
\BM^{(i)} =  \frac{1}{w_j^{2}} 
 \frac{S_{2,-j}^{(*,j)}}{Z^{(*,j)}_{ij}/u_n}
\boldb, \ i=1,\ldots,   N_n^{(j)}.
$$
For notational simplicity we continue the discussion with $j=1$, and in this case these 
 are essentially univariate iid random variables with the distribution of
\begin{equation} \label{e:T}
T_n=\frac{a_{22}Z_2+\cdots+a_{2p}Z_p}{w_1^2Z_1/u_n} 
\end{equation}
given
\begin{equation} \label{e:conditionT}
(a_{11}Z_1+\cdots+a_{1,p}Z_p)^2+(a_{21}Z_1+\cdots +a_{2p}Z_p)^2>u_n^2,
\ \ Z_1>u_n/w^{1/\alpha}\,.
\end{equation}
 Finally, we let $F_{T_n}$ denote the conditional law of $T_n$ in \eqref{e:T} given
the conditions in \eqref{e:conditionT}.  
For technical reasons we require further conditions on the latent factors in our subsequent results. We assume that the generic noise variable $Z$ in \eqref{e:X} and \eqref{e:regvar.Z} is positive or symmetric, and has a 
  probability density function $f_Z$ such that
\begin{equation} \label{e:temp.ass}
  f_Z(z) \ \ \text{is bounded away from 0 on compact intervals and bounded from above,}
\end{equation}
and 
\begin{equation} \label{e:density.bounds}
	B^{-1}z^{-(\alpha+1)}\leq f_Z(z)\leq B z^{-(\alpha+1)},
\end{equation}
	$\alpha>1$, for all $z\geq z_0$, some $B\geq 1$. 
 
 The following lemma shows that the conditional density function $f_{T_n}(t)=\partial F_{T_n}(t)/\partial t$ enjoys some  useful regularity properties.

\begin{lemma}\label{l:densityT}
Assume \eqref{e:temp.ass} and \eqref{e:density.bounds}. For $\alpha>1$ the density function $f_{T_n}$ is such that:
\begin{enumerate}
\item[(i)] There exists an $G\in(0,\infty)$ such that for all $n$ large enough, $f_{T_n}(t)\leq G$  for all  $t$.
\item[(ii)] $f_{T_n}$ is uniformly bounded from below on compact intervals, uniformly in $n$.
\item[(iii)] There is a constant $D\geq 1$ and a number $t_0\geq 0$ such that 
$ D^{-1}t^{-(\alpha+1)}\leq f_{T_n}(t)\leq Dt^{-(\alpha+1)}$ uniformly for all $n$ large enough and all $t\geq t_0$. 
\end{enumerate}
\end{lemma}

It is clear that an analogous result holds for the appropriate conditional densities in the symmetric case. 

The following intermediate result is the key ingredient for completing our analysis of the connected component of the extremal $k_n$-NN graph, at least in the case $d=2$, assuming certain regularity conditions on the noise variables.

\begin{lemma}\label{l:intervals}
Assume \eqref{e:temp.ass}, \eqref{e:density.bounds} and let $d=2$, $\tau>1$ and
consider the random variable $m_n$ defined by 
\begin{equation} \label{e:m.n}
m_n=N_n^{(1)}/\lceil \tau\log  N_n^{(1)}\rceil,  \ \text{so that by \eqref{e:In.size}} \ \ 
m_n\sim \frac{cw_1nu_n^{-\alpha}}{\tau \log (nu_n^{-\alpha})}, \ \ n\to\infty. 
\end{equation}
Define the intervals
\begin{equation} \label{e:I.in}
I_{i,n} = \left( F_{T_n}^{-1}\bigl( (i-1)/m_n\bigr),
  F_{T_n}^{-1}\bigl( i/m_n\bigr) \right), \ i=1,\ldots, m_n
\end{equation} 
as well as  intervals
along  vector $\boldb$ by
\begin{equation*} \label{e:J.in}
  J_{i,n}=I_{i,n}\boldb, \ i=1,\ldots,   N_n^{(1)}\,.
\end{equation*}
 Then, on an event with probability tending to one, there is a finite number $K_0=1,2,\ldots$ such that for all $n$ large enough and all $i=2,\ldots,
  m_n-I_0$, every point in $J_{i,n}$ is closer to every point in the
  intervals $J_{i-1,n}$ and $J_{i+1,n}$ than to any point in an
  interval $J_{i_1,n}$ with $|i-i_1|>K_0$.
\end{lemma}

The proofs of these two lemmas are contained in the Appendix.

We are now ready to state the main result showing that the extremes generated from the same underlying factor will also belong to the same connected component of the extremal $k_n$-NN graph with probability tending to one, under appropriate regularity conditions. This time, for simplicity, we only consider the symmetric case. 

\begin{theorem} \label{pr:d2.same}
	Assume \eqref{e:temp.ass}, \eqref{e:density.bounds} and let $d=2$. Then, if $k_n > {G}\log n$  with large enough ${G}>0$, 
 there is a sequence $({B}_{n,2})$ of events with $\bbP({B}_{n,2})\to 1$ as
 $n\to\infty$ such that, 
for all $n$ large enough, on the event ${B}_{n,2}$, any two points 
$\BY_{i_1}^{(j)}$ and 
$\BY_{i_2}^{(j)}$, $ i_1=1,\ldots, N_n^{(j)}$, $
i_2=1,\ldots, N_n^{(j)}$ will belong to the same 
connected component of the $k_n$-NN graph for any $j=1,\ldots, p$.  
\end{theorem}

The proof of the theorem has been relegated to the appendix.


It follows from Proposition \ref{pr:disconn.outside} and Theorem \ref{pr:d2.same} that, with probability tending to one as $n\to\infty$, the extremal $k_n$-NN graph obtained from a sample drawn from \eqref{e:X} will have exactly $m\le p$ connected components corresponding 
 to the $m$ distinct asymptotic point masses \eqref{eq:ck} of the model. In other words, the extremal $k_n$-NN graph consistently identifies the underlying clusters through its connected components. This in turn implies that spectral clustering will be consistent by Proposition \ref{prop:Lap_spectrum}. We have therefore shown the following main practical result.
 
\begin{corollary}\label{cor:consistency}
    Assume \eqref{e:temp.ass}, \eqref{e:density.bounds}, $d=2$, $k_n=o(nu_n^{-\alpha})$ and $k_n> G\log n$. Then, spectral clustering will consistently identify the clusters of extremes arising from the linear factor model.
\end{corollary}

\begin{remark}
    In practice consistent clustering can be achieved by taking $k_n>  G_0\log N_n$ for some $G_0>0$ and $k_n=o(N_n)$ since $N_n/(nu_n^{-\alpha})\cip cw$, as $n\to\infty$. In our experiments we chose $k_n=\lceil\frac{N_n}{\tau\log N_n}\rceil+1$ for some $\tau>1$.
\end{remark}

Corollary  \ref{cor:consistency} suggests a simple strategy for estimating the asymptotic angular measure of the extremes generated from the linear factor model \eqref{e:X}. Assume we run spectral clustering on  the extremal $k_n$-NN graph. Then we can denote by $\hat{\mathcal{I}}_n^{(j)}$ the set of indices corresponding to the $j$th cluster found by the algorithm for $j=1,\dots,m$. With these sets we can define $\hat{N}_n^{(j)}=\mbox{card}(\mathcal{I}_n^{(j)})$ and  estimate the centers of the spectral measure and their respective masses as
\begin{equation}
    \label{eq:ck_estimates}
\hat\bc_j=\frac{1}{\hat{N}_n^{(j)}}\sum_{i\in\hat{\mathcal{I}}_n^{(j)}}\frac{\bX_i}{\|\bX_i\|} \quad \mbox{ and } \quad \hat\pi_j=\frac{\hat{N}_n^{(j)}}{N_n},\quad j=1,\dots,p.
\end{equation}
The following result is an inmediate consequence of the main results of this section.
\begin{corollary}
\label{rem:consistent_estimation}
Suppose $m=p$ and that the conditions of Proposition \ref{pr:disconn.outside} and Theorem \ref{pr:d2.same} hold. Then,  $\hat\bc_j\cip \bc_j$ and $\hat \pi_j\cip w^{-1}\|\ba^{(j)}\|^\alpha$ for all $j=1,\dots,p$.
\end{corollary}
Note that in practice one can also normalize the estimates $\hat\bc_j$ to ensure that they lie in the unit sphere for all $n$. Clearly the resulting estimators remain consistent under the conditions of Corollary \ref{rem:consistent_estimation}.

 Even though the theoretical results of this section use the assumption $\alpha >1$ in \eqref{e:density.bounds}, we believe they should also hold when $\alpha\in (0,1]$. The numerical experiments shown in the next section supports this assertion.

\section{Numerical illustrations} \label{sec:numerics}

 In all the examples considered below we compute weighted adjacency matrices using the exponential kernel $d(\bx,\by)=\exp(-s\|\bx-\by\|)$ with $s=1$ and select the number of clusters as suggested by the screeplots of the fully connected weighted adjacency matrices ${W}$. It matched nicely the correct number of clusters, when known. We consider sample sizes $n=\{1000,5000,25000,125000\}$ and take a sample of extremes corresponding to observations whose Euclidean norm is larger or equal to the following vector of corresponding sample quantiles:  {$\beta=\{0.9,0.96,0.984,0.9968\}$}, respectively. These quantiles were chosen to lead to samples of extremes of sizes $N_n=\{100,200,400,800\}$, correspondingly.  For these extremes we define $k_n$-nearest neighbors graphs with ${k_n=\lceil\frac{N_n}{C\log N_n}\rceil+1}$, the corresponding values of the constants are in the vector $ C=\{3,5,7,9\}$  .

\subsection{Linear factor model with and without noise} \label{sec:lfm}
 
As a first example, consider $d-$dimensional vectors that follow the $p-$dimensional linear factor model 
 \begin{equation}\label{eq:LFMnoise}
     \bX= A \bZ+\sigma\bm{\varepsilon}\,,
 \end{equation}
 where $A\in \mathbb{R}^{d\times p}$ is a matrix of factor loadings, $\bZ=(Z_1,\ldots,Z_p)^\top$ is a $p$-dimensional vector consisting of iid standard Fr{\'e}chet distributed components ($\alpha=1$), $\sigma\geq0$ regulates the signal to noise ratio and  $\bm{\varepsilon}$ is a noise vector obtained by multiplying a univariate independent  standard Fr{\'e}chet with an independent $p$-dimensional random vector of   iid standard normals, i.e.,
 \begin{equation}\label{eq:noiseterm}
     \bm{\varepsilon}= \bN \eta \,,
 \end{equation}
 where $\eta$ is standard Fr{\'e}chet, $\bN=(N_1,\ldots,N_p)^\top$ is a  $p-$random vector consisting of iid standard normals, and $\bZ,\eta$, and $\bN$ are independent. Now using computations similar to those given in Section \ref{sec:spectrum}, it can be shown that 
 \begin{eqnarray} 
     \frac{\bbP(\|\bX\|>x)}{\bbP(Z_1>x)}&\sim &\frac{\bbP(\sum_{i=1}^p\|\ba^{(i)}\|^2Z_i^2 +\sigma^2\|\bN\|^2\eta^2>x^2)}{\bbP(Z_1>x)} \nonumber\\
     &\sim &\frac{\sum_{i=1}^p\bbP(\|\ba^{(i)}\|^2Z_i^2>x^2)+\bbP(\sigma^2\|\bN\|^2\eta^2>x^2)}{\bbP(Z>x)}\nonumber\\
     &\to &\sum_{i=1}^p\|\ba^{(i)}\| +\sigma\bbE\|\bN\|\,,~~~\mbox{as $x\to\infty$,}\label{eq:normLFM}
 \end{eqnarray}
 where the last line follows from an application of Breiman's lemma, see \cite{breiman1965some}. Taking this calculation one step further, we find that the angular measure associated with the model is  \eqref{eq:LFMnoise} (see \eqref{e:spectral.m}),
 \begin{equation} \label{eq:specLFMnoise}
  \Gamma(\cdot) =w^{-1}\left(\sum_{i=1}^p \|\ba^{(i)}\|
       \delta_{\frac{\ba^{(i)}}{\|\ba^{(i)}\|}}(\cdot)+\sigma\bbE\|\bN\|
       \delta_{\frac{\bN}{\|\bN\|}}(\cdot)\right)\,,
\end{equation}
where $w=\sum_{i=1}^p\|\ba^{(i)}\| +\sigma\bbE\|\bN\|$.  In other words, $\Gamma$ has discrete mass points  at $\frac{\ba^{(i)}}{\|\ba^{(i)}\|}$ with probability $\|\ba^{(i)}\|/w$, $i=1,\ldots,p$ and a uniform distribution $\bN/\|\bN\|$ on  $\bbs^{d-1}$ with probability $\sigma\bbE\|\bN\|/w$.  This latter piece corresponds to the noise component $\sigma \epsilon$.  So the goal here is to identify the discrete components of $\Gamma$ using our method when the model does not strictly follow the LFM.  
Figure \ref{fig:scatterLFM}  shows pairwise scatter plots of the angular components of extremes generated from a pure signal and a noisy LFM with $\sigma>0$.

 \begin{figure}[h!]
    \hfill
\subfigure[Pure signal LFM]{\includegraphics[scale=0.42]{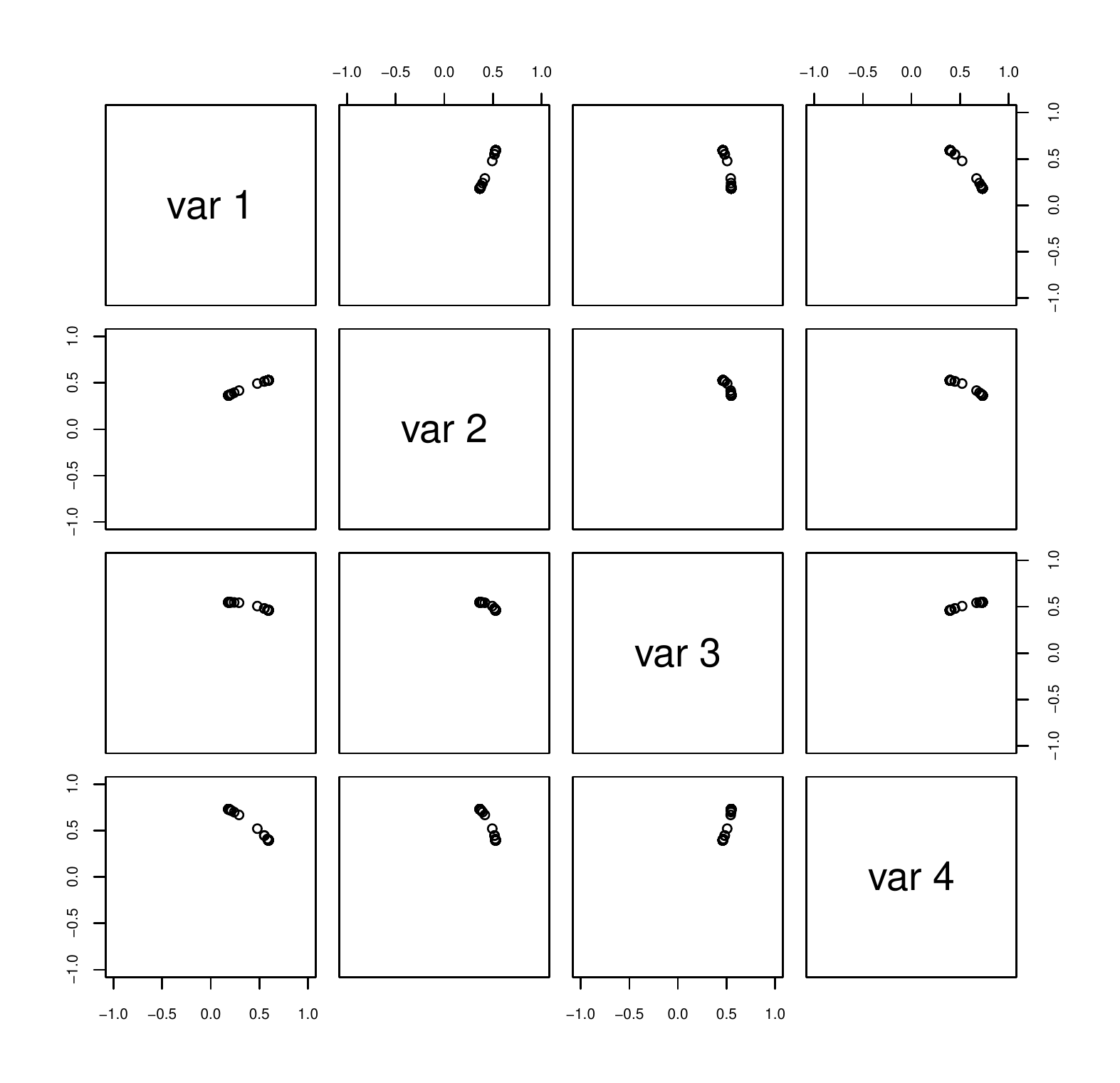}}
\hfill
\subfigure[Noisy LFM]{\includegraphics[scale=0.42]{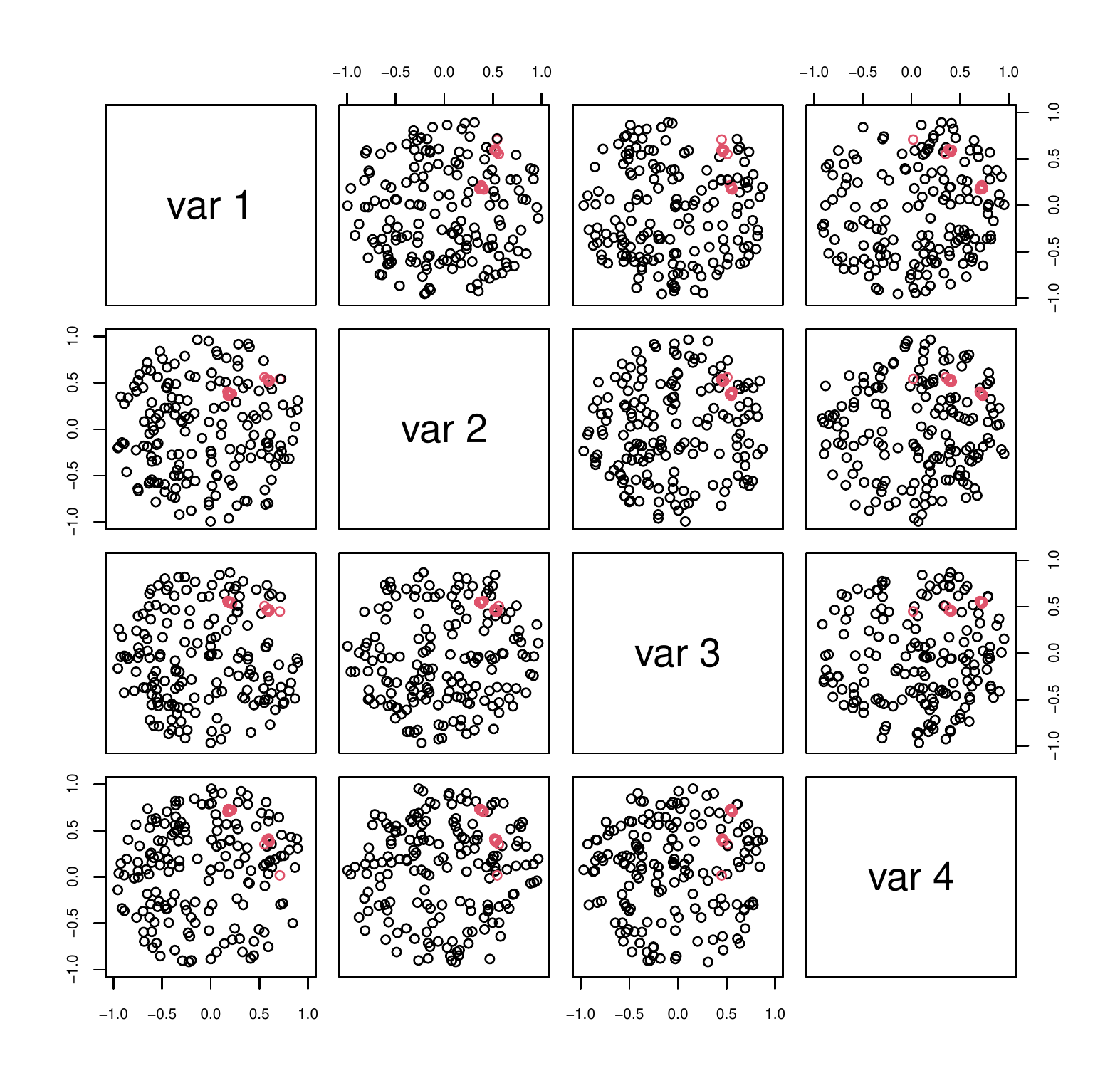}}
\hfill    \caption{{\small Pairwise scatterplots of the angular part of the extremes generated from \eqref{eq:LFMnoise} with factor loading matrix  \eqref{eq:A}, $n=125000$, $N_n=400$ and $\sigma=\{0,1\}$.  In both cases there are two clear clusters corresponding to the signal. The red points in subfigure (b) denote extremes attributed to the signal $A^\top \mathbf{Z}_i$. }}
    \label{fig:scatterLFM}
\end{figure}

 We note that if $\sigma=0$, then  model \eqref{eq:LFMnoise} is approximately equal to the max-linear model $X=(\vee_{j=1}^ka_{1j}Z_{j},\dots,\vee_{j=1}^ka_{pj}Z_{j})^\top$ and will in fact have the same asymptotic spectral measure. 
 Intuitively, this model generates $p$ clusters of extremes since the noise term is only adding uniform noise to the angular measure.  
 
 As part of a simulation study, we consider  $\sigma=\{0,1, 3, 5\}$ and choose
 \begin{equation}\label{eq:A}
    A= \begin{pmatrix} 
 0.1 & 0.9\\
 0.2 & 0.8\\
 0.3 & 0.7 \\
 0.4 & 0.6 \\
 \end{pmatrix}.
 \end{equation}
This model is similar to one of the max-linear models considered in the simulations of  \cite{janssenandwan2020} where our  factor loading matrix $A$ can be viewed as a deterministic version of their random factor loadings. In the simulations we took two clusters  for the pure signal model where $\sigma=0$ and three clusters for the noisy model when $\sigma >0$ as these values are suggested by the typical screeplots we observed; see  Figure \ref{fig:screeplotsMaxN}.

We compute the normalized columns of the $A$ matrix which correspond to the location of the point masses of the spectral distribution (these are the $\bc_k,\,  k=1,2$ in \eqref{eq:ck}).  After applying our method to a single realization of size $n=125000$ with $N_n=400$, $k_n=\lceil \frac{400}{5\log 400}\rceil+1=15$,  visualized in the pairwise scatter plots of Figure \ref{fig:scatterLFM},
 we obtained the estimates of the $\bc_k$  represented in 
Figure \ref{fig:heatmap_LFM}.   These masses on the sphere are estimated by taking the mean of all members in each of the identified clusters,  seen  in Figure \ref{fig:clustersLFM}, and then normalizing it to lie on the unit sphere.  
The two panels in Figure \ref{fig:heatmap_LFM} correspond to the  cases of 2 clusters and no noise and two clusters with uniform noise.  
In the first plot, the heat map does a good job in recreating the relatives size of the mass locations.  In the second panel, the first two columns of the matrix, also reproduce the relative sizes of the columns (increasing in the first and decreasing in the second) of the $A$ matrix.  The third column corresponds to the cluster of  points that have not been assigned to either of the first two clusters.  As such they are essentially scattered uniformly around the unit sphere but away from the locations of the point masses corresponding to the first two columns.  This is reflected in the third cluster having more negative values as indicated by the softer (red colors) in the heat map.

\begin{figure}[h]
    \centering
    \includegraphics[scale=0.65]{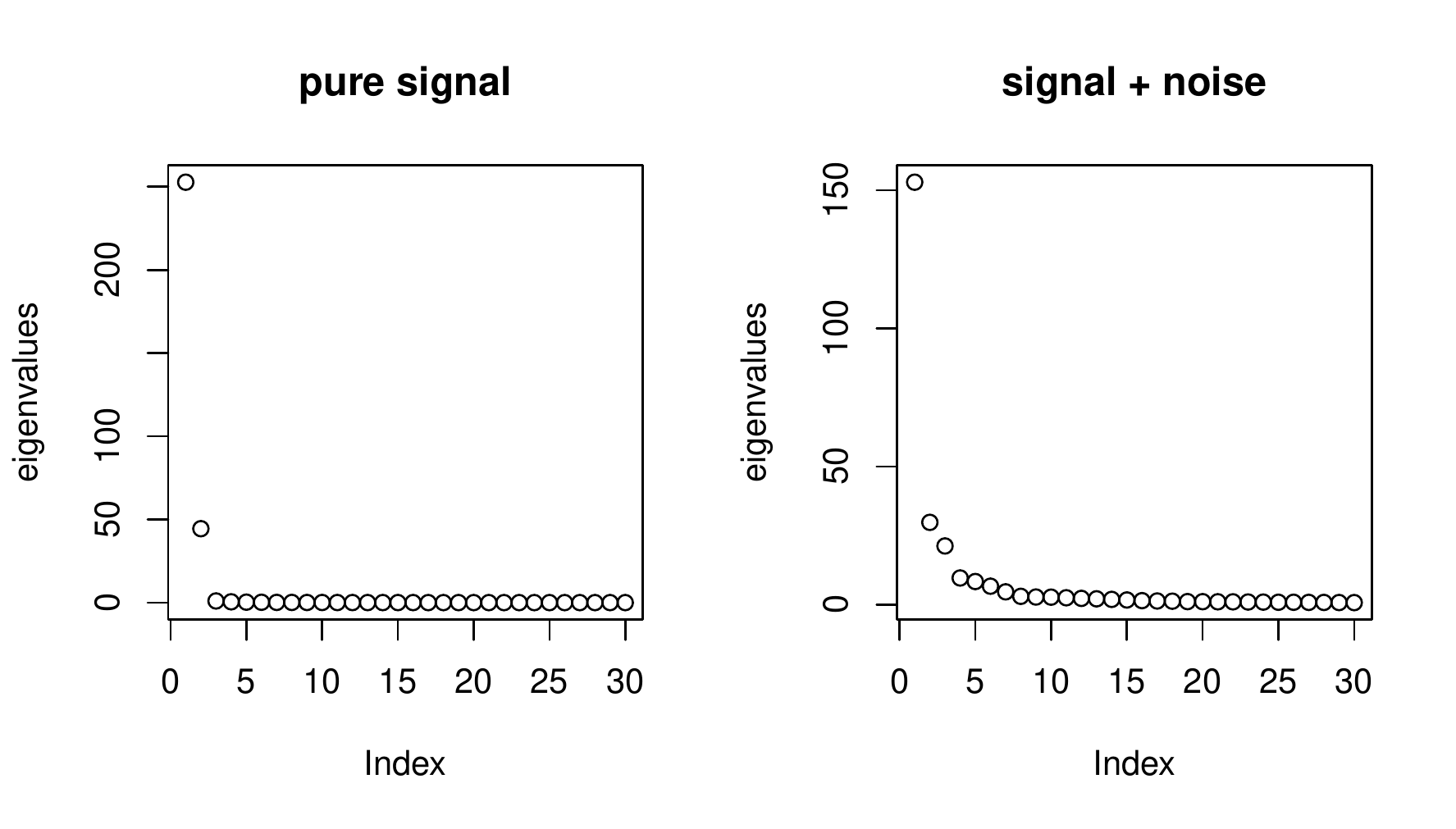}
    \caption{{\small Screeplots of fully connected kernel matrix of pure signal and noisy linear factor models noise models.}}
    \label{fig:screeplotsMaxN}
\end{figure}
 \begin{figure}[h]%
\hfill
\subfigure[pure signal]{\includegraphics[scale=0.35]{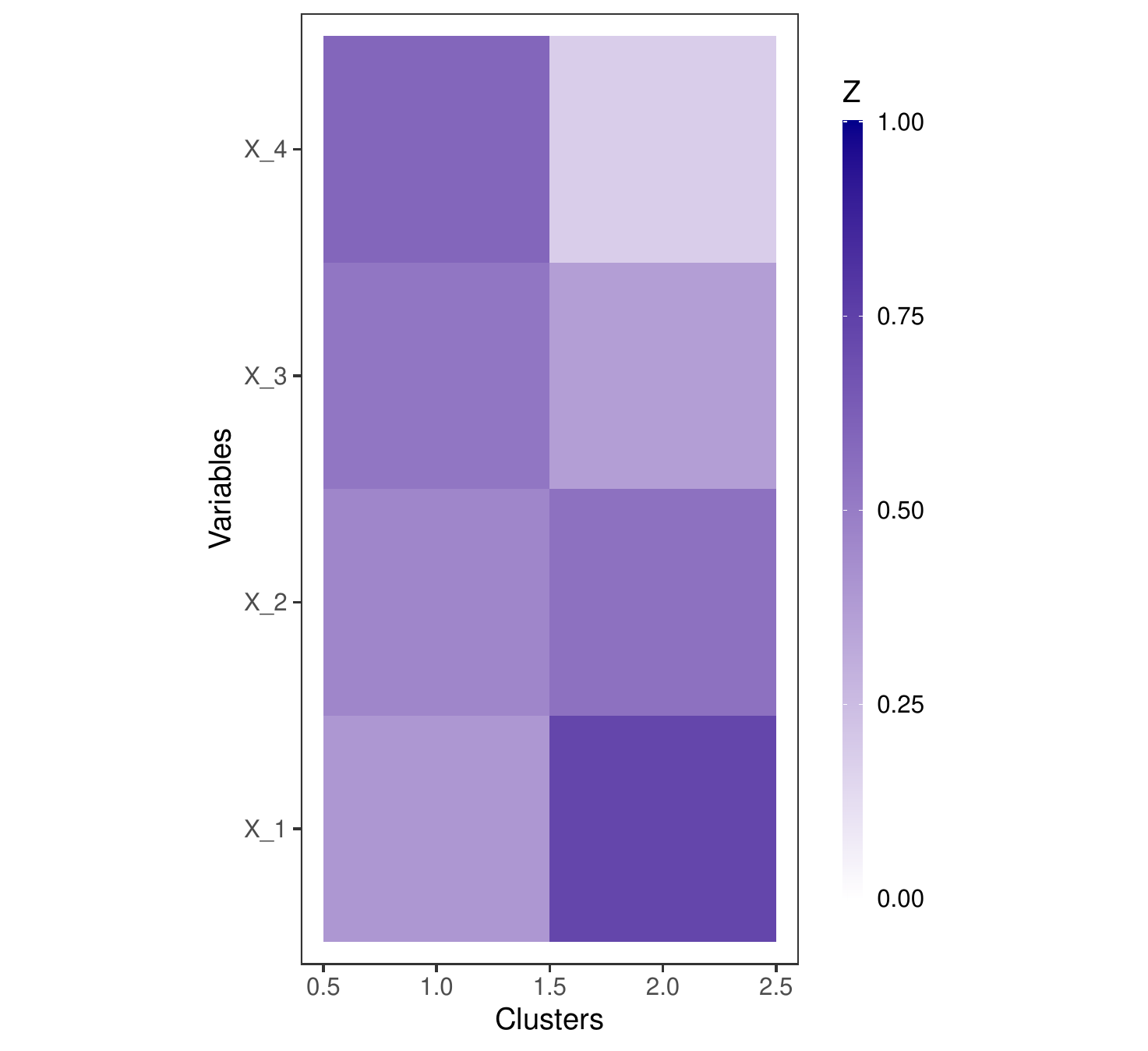}}
\hfill
\centering\subfigure[uniform noise on the sphere]{\includegraphics[scale=0.35]{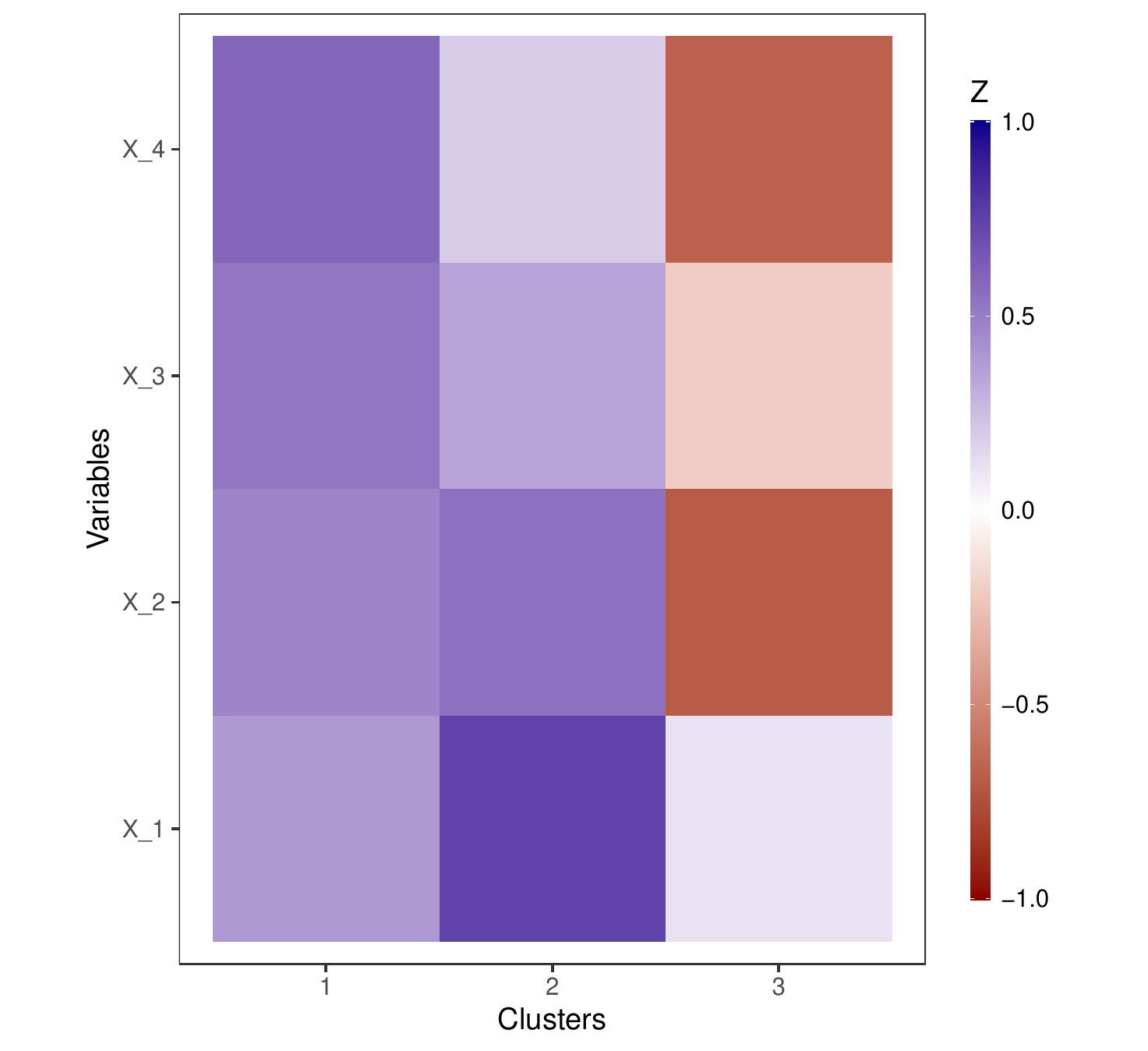}}
\caption{The heat maps show the estimated cluster centers based on the cluster assignments displayed in Figure \ref{fig:clustersLFM}. The extremal sample corresponds to four dimensional extremes generated from LFM given by \eqref{eq:LFMnoise} with loading matrix \eqref{eq:A} and $\sigma=0$ and $\sigma=1$ respectively. In both cases we took $n=125000$, $N_n=400$ and $k_n=15.$} 
    \label{fig:heatmap_LFM}
\end{figure}
\begin{figure}[H]%
\hfill
\subfigure[Pure signal LFM]{\includegraphics[scale=0.45]{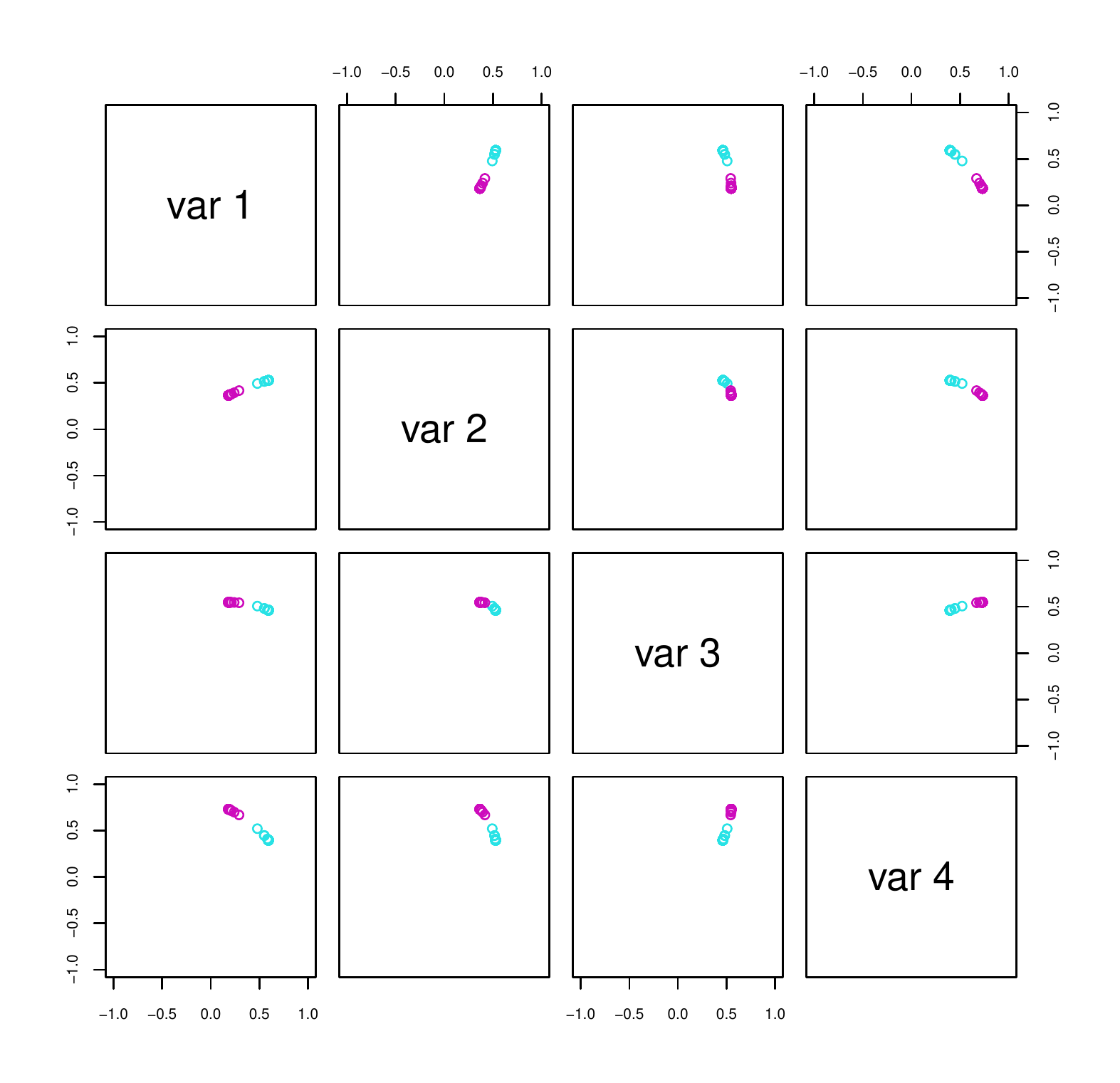}}
\hfill
\subfigure[Noisy LFM]{\includegraphics[scale=0.45]{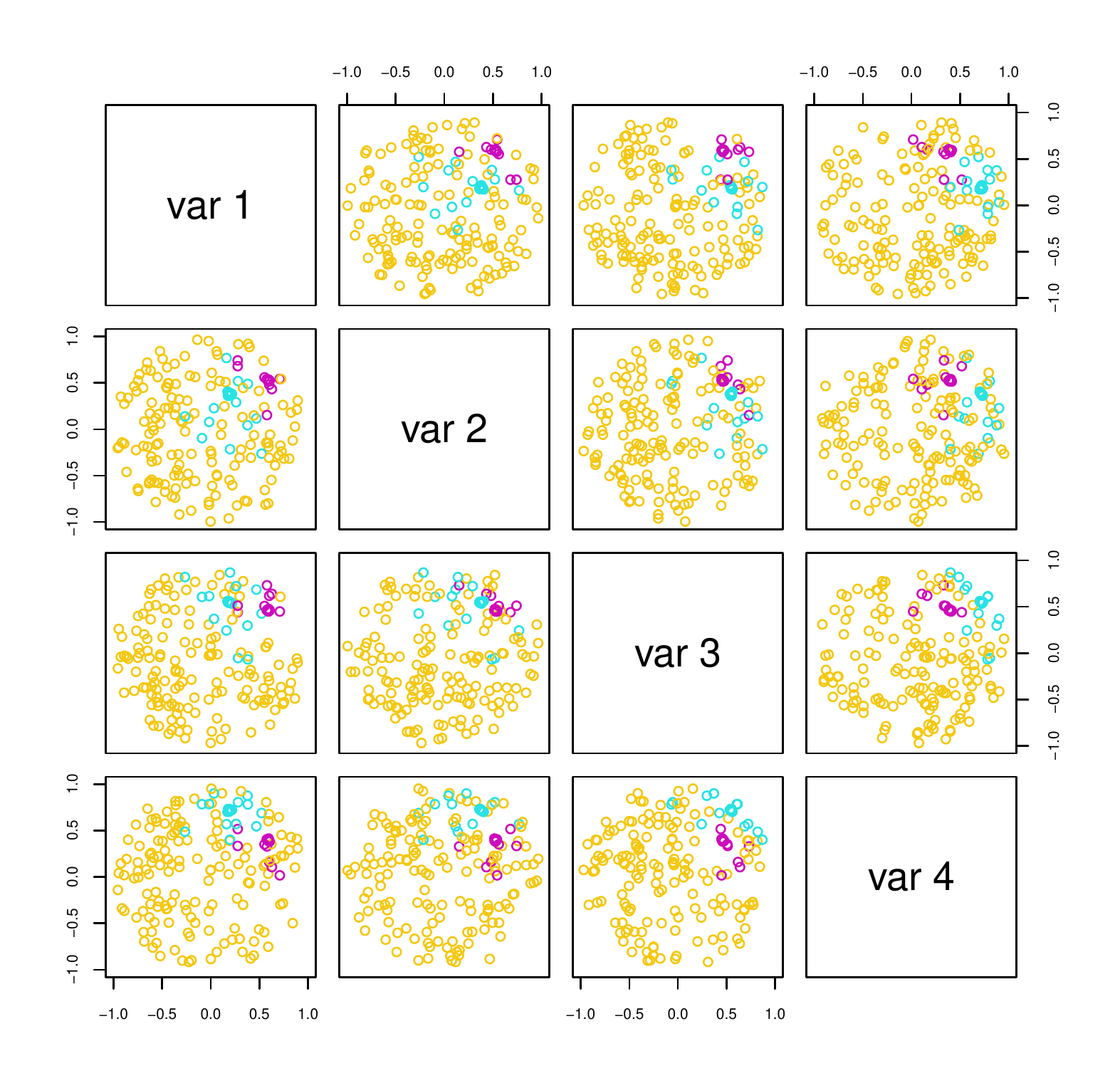}}
\hfill
\caption{Cluster assignments output of spectral clustering applied to data generated from \eqref{eq:LFMnoise} with $n=125000$, $N_n=400$ and $\sigma=\{0,1\}$. In both cases spectral clustering used an extremal $15$-NN graph.}
    \label{fig:clustersLFM}
\end{figure}

A small simulation study was conducted for this LFM model with and without noise.  Based on the screeplots, we used 2 clusters in the noiseless case and 3 clusters in the case with noise.  The two normalized columns of the $A$ were estimated and the boxplot of the estimation error measured in Frobenius norm are displayed in Figure \ref{fig:conv_AML} for $\sigma=0$ and in Figure \ref{fig:conv_nAML} for the case $\sigma>0$. The succession of boxplots in each row correspond to an increasing $N_n$, with the centers and width becoming smaller.  Note that the scales on the plots change across the row.  The boxplots in blue correspond to our method with difference choices of nearest neighbors as a function of $C$, and the yellow boxplot is based on the spherical $k$-means approach considered in \cite{janssenandwan2020}.  In the $\sigma=0$ case, our method performs about the same or slightly better than the spherical $k$-means method.  However, as one adds noise to the model, our method generally outperforms spherical $k$-means.     
In models with larger noise, it can be more difficult to estimate the LFM signal. So to compare performance across difference sample sizes and level of noise, we can calibrate by calculating a notion of  {\it signal to noise ratio}.  In this context we consider the part of the mass in the angular measure associated to the {\it signal} in \eqref{eq:LFMnoise}, which as a function of $\sigma$ is given by
$$
\mbox{SNR}(\sigma):=\frac{\sum_{i=1}^p\|\ba^{(i)}\|}{\sum_{i=1}^p\|\ba^{(i)}\| +\sigma\bbE\|\bN\|}\,.
$$
In the absence of any noise, i.e., $\sigma=0$, then SNR is 1 while as $\sigma\to\infty$, SNR converges to 0.  For the simulation example above for which $d=4, p=2$, we have $\bbE\|\bN\|=\sqrt 2\Gamma(5/2)/\Gamma(2)=1.880$.  Hence SNR$(\sigma)=2.065/(2.065+\sigma1.880)$.  In looking at the various plots in Figure \ref{fig:conv_nAML}, it is instructive to compute the {\it effective sample size} given by ESS$=\mbox{SNR}\times N_n$.  This number essentially gives the expected sample size of the number of observations, from the total $N_n$, that {\it come from the signal}.  With this index in mind, plots that have the same ESS values (reported in the caption of Figure \ref{fig:conv_nAML})  generally show similar results since the procedures are applied to the roughly the same number of extreme observations attributed to the signal component in the model.  We finally note that in this simulation $\alpha=1$ which is not currently covered by our LFM theory but is the setting proposed in the simulations of \cite{janssenandwan2020}. We carried out simulations with $\alpha=0.5$ and $\alpha=2$ and obtained qualitatively the same type of results as the ones reported here. The only noticeable difference was that spherical $k$-means  seems to work better with larger $\alpha$ in the noisy model, but is much worse for small $\alpha$. In both cases spectral clustering outperformed spherical $k$-means.

\begin{figure}[h!]
    \centering
    \includegraphics[scale=0.55]{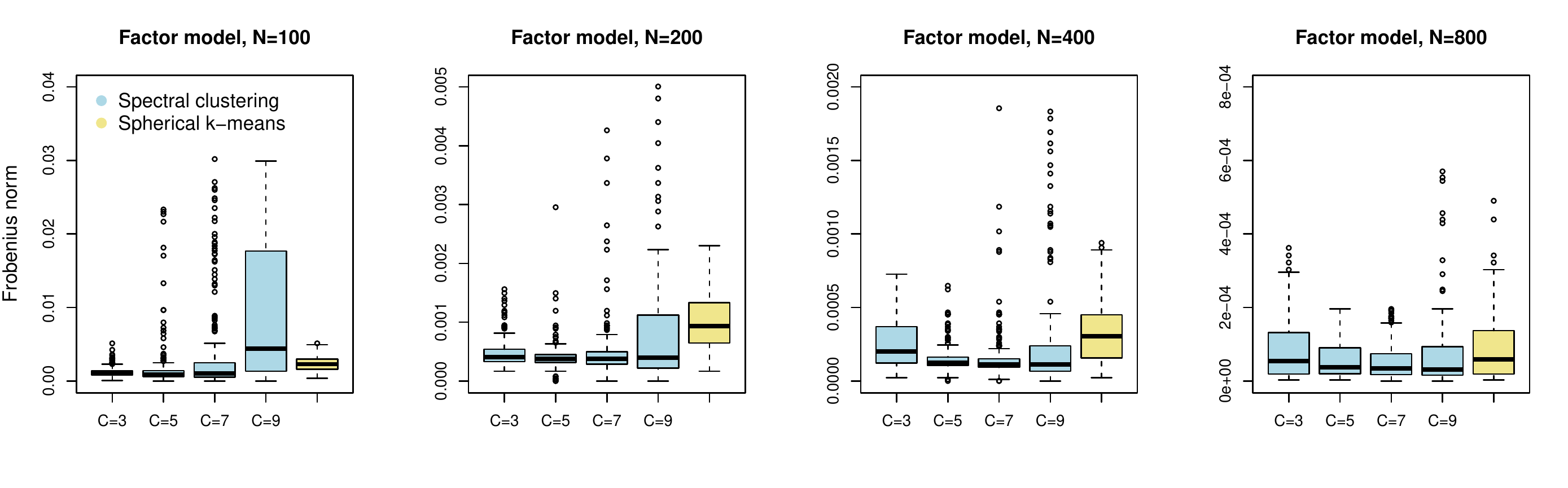}
    \caption{{\small Estimation error measured in Frobenius norm when $\sigma=0$.}}
    \label{fig:conv_AML}
\end{figure}

\begin{figure}[H]
    \centering
    \includegraphics[scale=0.55]{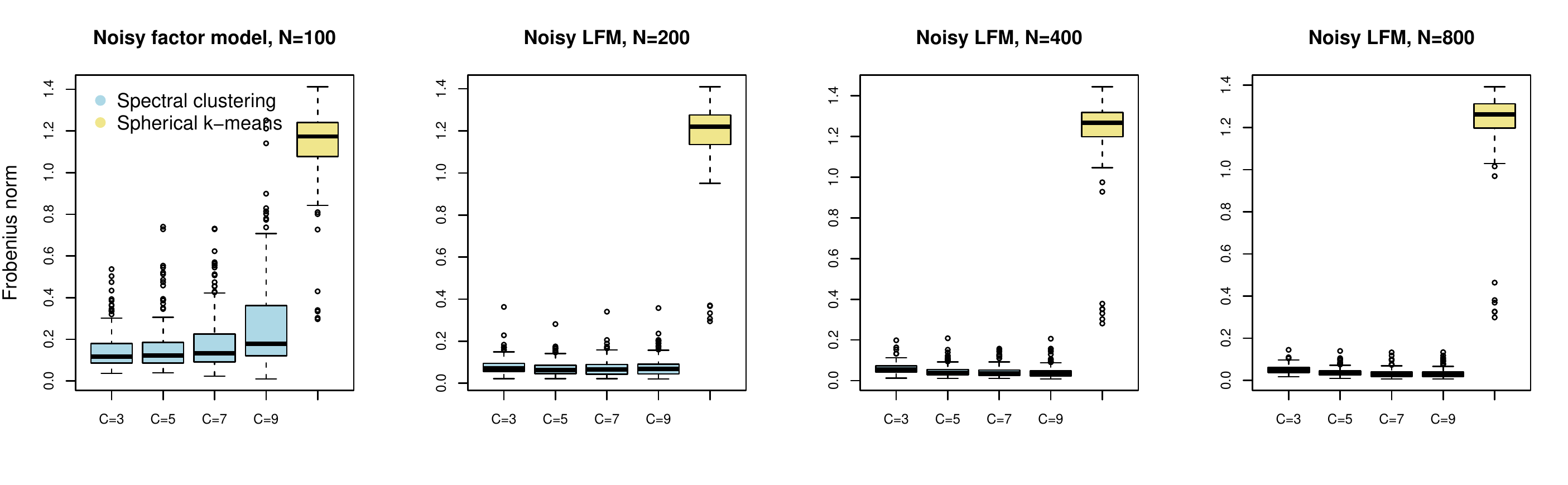}
    \includegraphics[scale=0.55]{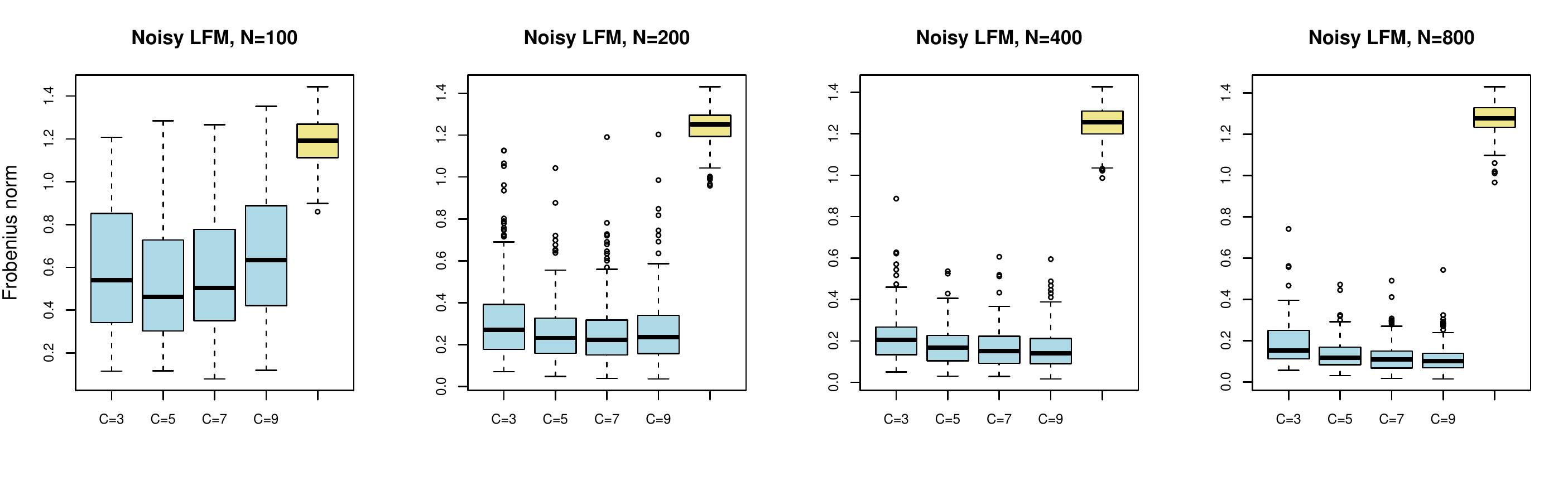}
    \includegraphics[scale=0.55]{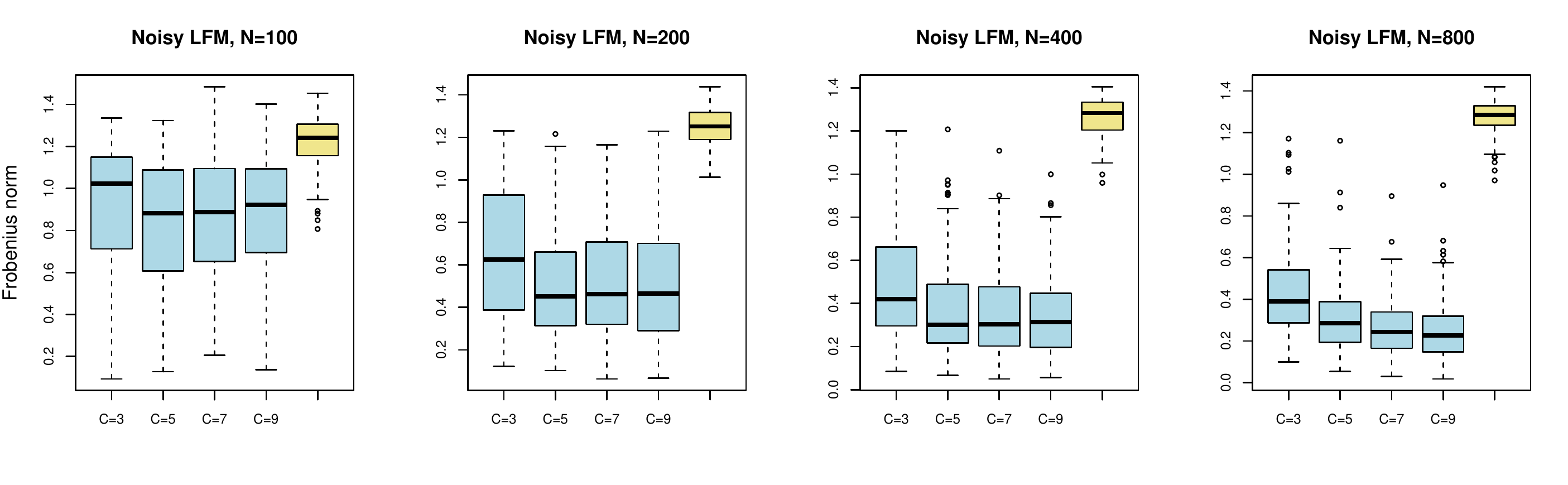}
    \caption{{\small Estimation error of cluster centers measured in Frobenius norm. data was generated from the noisy LFM \eqref{eq:LFMnoise}. The sample sizes increases from left to right as $n=\{1000,5000,25000,125000\}$ and $N_n=\{100,200,400,800\}$, and from noise level increases from the top down as  $\sigma=\{1,3,5\}$. Across rows the ESS are: $\{52,105,209,418\}$, $\{27, 54, 107, 214\}$, $\{18, 36,72,144\}$}}
    \label{fig:conv_nAML}
\end{figure}

\subsection{Bivariate extremes from MA(3)}
 We consider the model discussed in the introduction and represented in Figure \ref{fig:MA3_data}. More specifically,  the model is $Y_t=Z_t+.5Z_{t-1}-.6 Z_{t-2}+1.5Z_{t-3}$, where $\{Z_t\}$ is an iid symmetric stable random variables  with  index $\alpha=1.8$. We analyze the extremal dependence structure of the  bivariate vector $\bX_t=(Y_t,Y_{t-1})^\top$ by   looking for clusters in the extremes of $\bX_t$. This model can be written in the form \eqref{e:X} since we can define $\bZ_t=(Z_{t},Z_{t-1},Z_{t-2},Z_{t-3},Z_{t-4})$ and hence
\begin{equation*}
 \bX_t =   \begin{pmatrix}
 1 & 0.5 & -0.6 & 1.5 & 0 \\
 0& 1 & 0.5 & -0.6 & 1.5
 \end{pmatrix}
 \bZ_t.
 \end{equation*}
 Note that even though in this case the sample $\{\bX_t\}$ is not independent, the asymptotic distribution obtained in Theorem \ref{l:weak.limit} still holds. In particular, the angular distribution is supported in the points \eqref{eq:ck} i.e.
\begin{align*}
&\bc_{1,\pm} = \pm(1,0), \quad \bc_{2,\pm}= \pm(.5,1)/\sqrt{1.25},\quad \bc_{3,\pm}=\pm(-0.6,0.5)/\sqrt{0.61}, \\
&
\quad\bc_{4,\pm}= \pm(1.5,-0.6)/\sqrt{2.61}
\quad \mbox{and}\quad \bc_{5,\pm}=\pm(0,1).
\end{align*} 
 
\begin{figure}[h]
    \centering
    \includegraphics[scale=0.4]{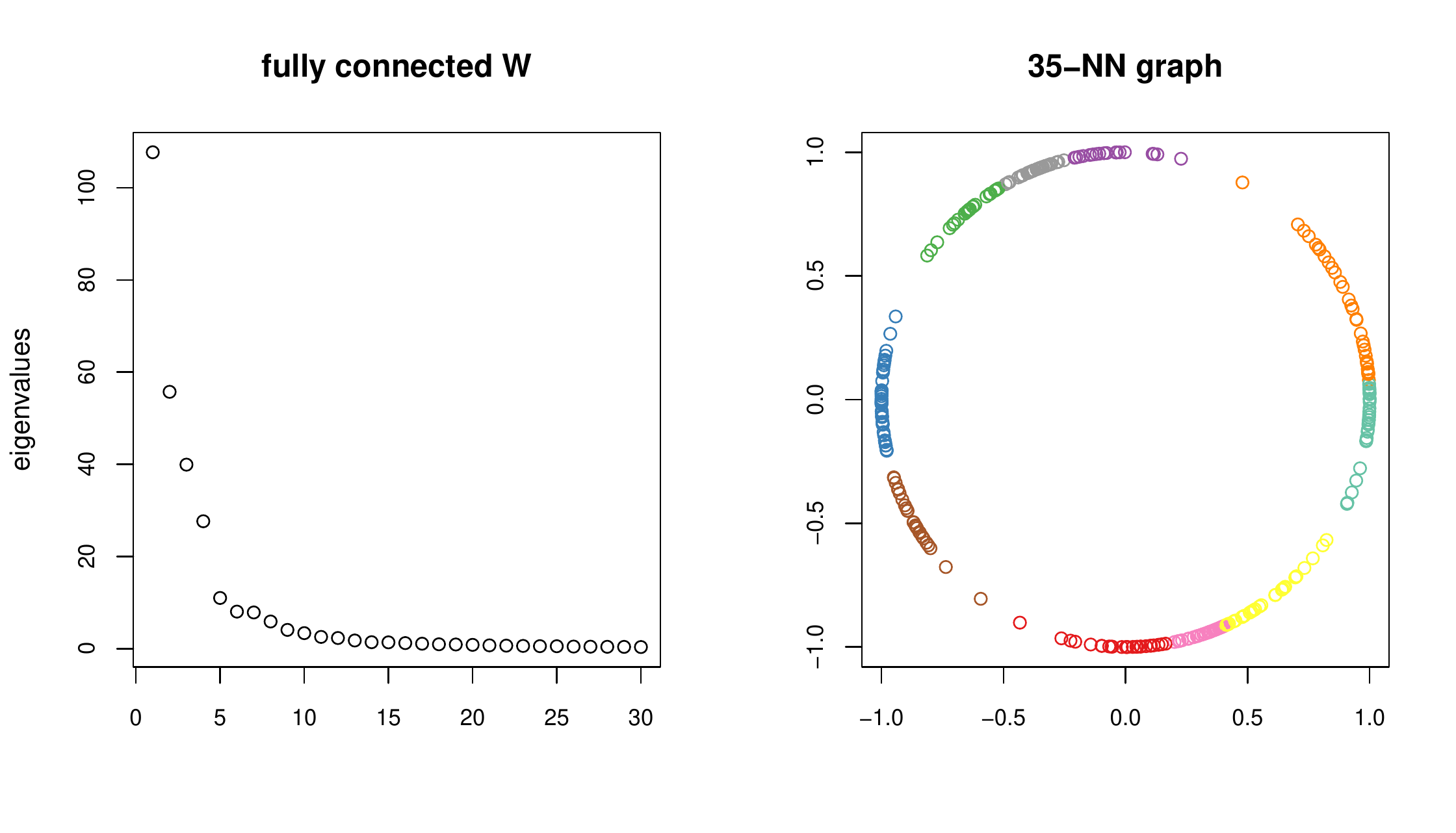}
    \caption{{\small Screeplot of kernel matrix and clustering performance of 2 dimensional MA$(3)$ extremes when $n=25000$ and $N_n=400$.}}
    \label{fig:MA3_plots}
\end{figure}
 \begin{figure}[h]
    \centering    
    \includegraphics[scale=0.5]{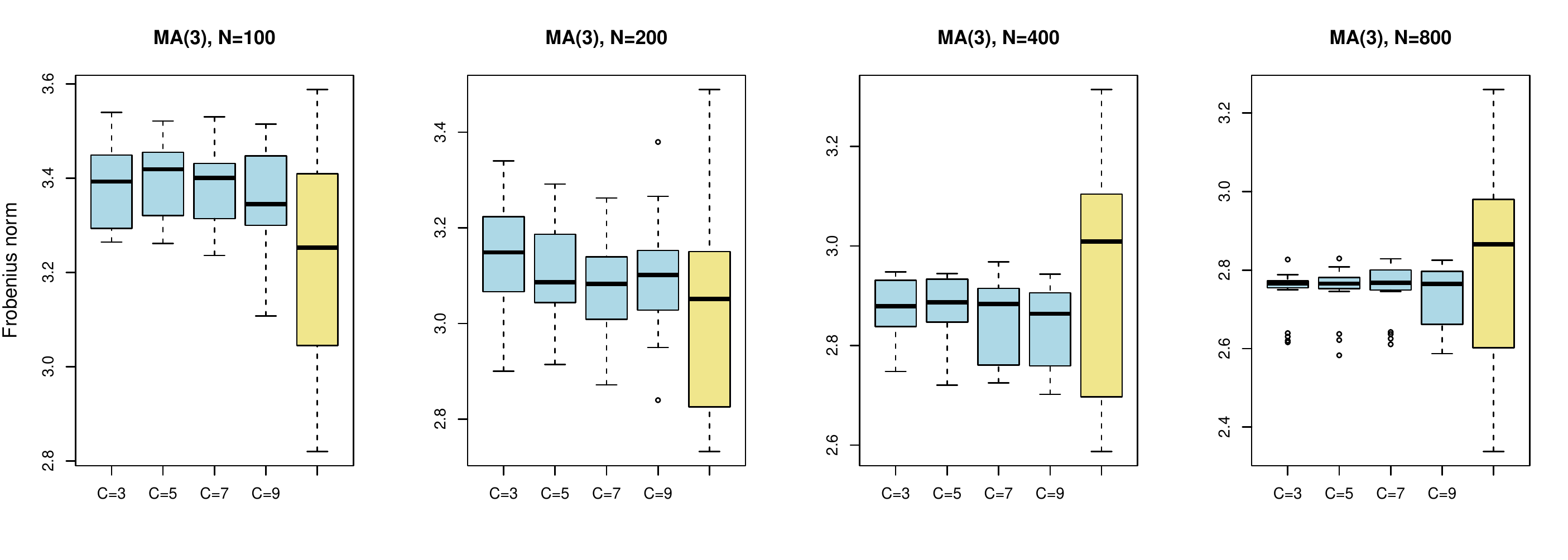}
 \caption{{\small Estimation error of the matrix of atoms of the spectral  measure of the symmetric MA(3) model. The sample sample sizes were $n=\{1000,5000,25000,125000\}$ giving   $N_n=\{100,200,400,800\}$}. We used $k_n=\lceil \frac{N_n}{C\log N_n}\rceil+1$ nearest neighbor graphs.}     \label{fig:MA3_conv2}
\end{figure}
Figure \ref{fig:MA3_plots} illustrates the behavior of spectral clustering for this model when $N_n=400$ and {$k_n=\lceil \frac{400}{2\log 400}\rceil+1=35$.} It is worth noting that in Figure \ref{fig:MA3_data} we had a larger sample size of $100,000$ and stricter quantile threshold of $0.998$ resulting in smaller number of observations considered as extremes, but with an empirical distribution visibly closer to the prescribed asymptotic discrete distribution. Therefore the simulation scenario considered here is more difficult. Figure \ref{fig:MA3_conv2} illustrates the convergence of the method.  While the spectral $k$-means method of \cite{janssenandwan2020} performs slightly better than our spectral clustering for $N_n\le 200$, our proposed method appears better with much smaller variability for a larger number of extremes.  The choice  $k_n$ of nearest neighbors did not appreciably impact the performance of spectral clustering across the different sample sizes.

\subsection{Air pollution data}

We revisit the data analyzed by \cite{heffernanandtawn2004} and \cite{janssenandwan2020}. It is available in the \texttt{R} package \texttt{texmex} and consists of daily measurements of five air pollutants in the city of Leeds, UK. It was collected between 1994 and 1998, and split into summer and winter months yielding a total of $578$ and $532$ observations respectively. Following standard practice in multivariate extremes data analysis we standardize the marginal distribution of the data to focus on the extremal dependence. More specifically, we transform the marginals of the original observations $\mathbf{X}_i$ as in \cite{janssenandwan2020} i.e., we let 
$$Y_{ij}:=1/\{1-F_{nj}(X_{ij})\}, $$
where $F_{nj}(x)=\frac{1}{n}\sum_{i=1}^n\mathbbm{1}(X_i\leq x)$ denotes the $j$th marginal empirical cumulative distribution function, $x\in\mathbb{R}$ and $j=1,\dots,d$. We then proceed to define the extremal observations as the $10$\%  of the transformed observations $\{\mathbf{Y}_i\}$ with largest Euclidean norm and analyze their angular components with our algorithm. We analyze this data using spectral clustering with the exponential kernel and $s=1$ as in the simulated data.
 The screeplots in Figure \ref{fig:screeplots_pollution} suggest that one should consider 5 clusters for this data.

\begin{figure}[h]
    \centering    
    \includegraphics[scale=0.65]{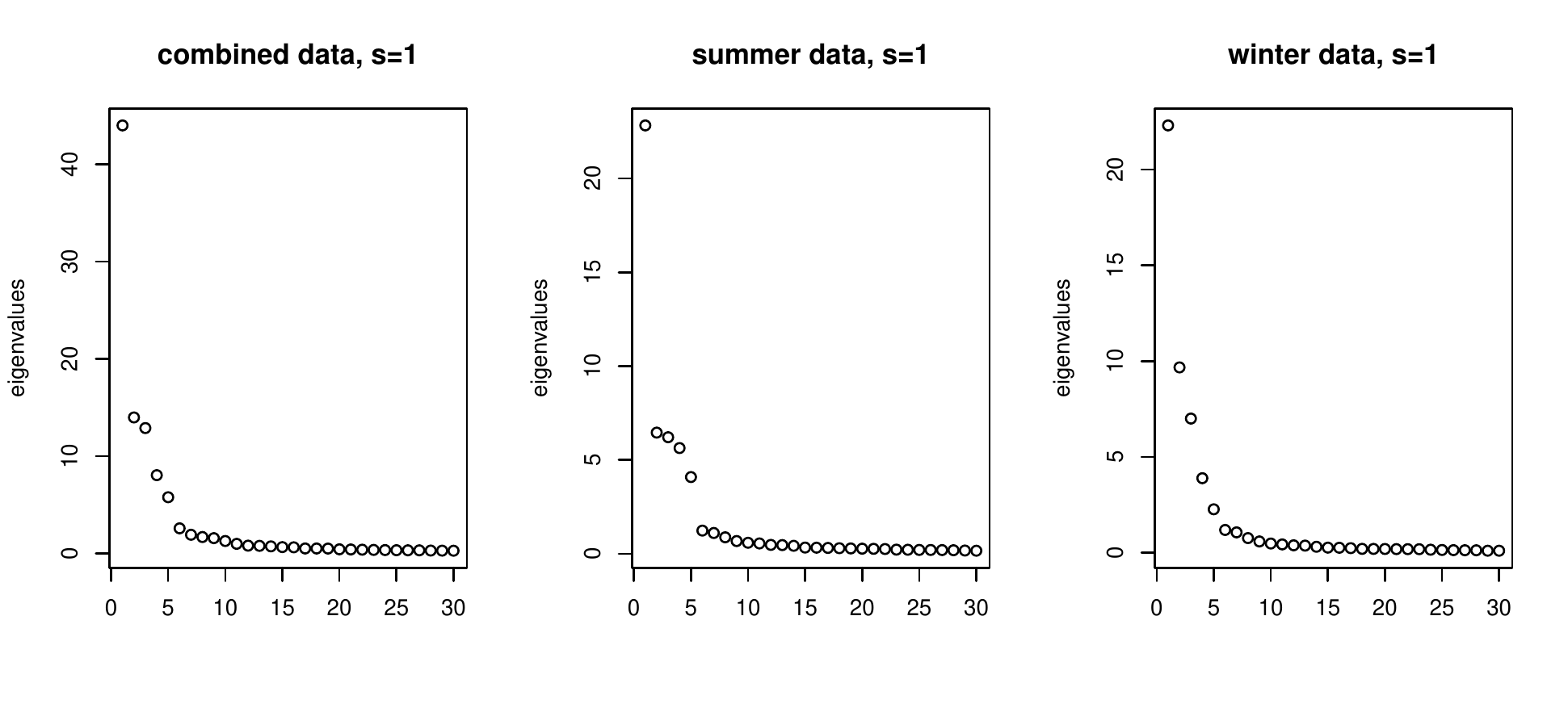}
 \caption{{\small Screeplots of the kernel matrix of the air pollution data extremes obtained with the exponential kernel and bandwidth parameter $s=1$ }. }    
 \label{fig:screeplots_pollution}
\end{figure}

Figure \ref{fig:5clusters_pollution}  shows the estimated cluster centers $\mathbf{c}_j$ for $j=1,\dots,5$. We note that the ``elbow plot'' considered by \cite{janssenandwan2020} suggested the authors to use 4 or 5 clusters in their article.  Our results for 5 clusters is consistent with their analysis. Specifically, the normalized cluster centers in the heat map of Figure \ref{fig:5clusters_pollution} show that the extremes of the five air pollutants act mostly independent.  Looking a bit more closely, both NO and NO2 share common strength in clusters 2 and 3, which is much stronger in winter than in summer.  PM10  also shares a common source (cluster 2) with NO and NO2, which is more pronounced in winter than summer.  For the O3 and NO2 pollutants, we examined time lagged dependence by applying the spectral clustering algorithm to the vector $\bX_t=(X_t,X_{t-1},X_{t-2},X_{t-3})^T$, where $X_t$ represents either the measured value of O3 or NO2 on day $t$.  The resulting heat plots for the cluster centers (4) are displayed in Figures \ref{fig:O3_pollution} (O3) and \ref{fig:NO2_pollution} (NO2). 
 The super and sub diagonals reflect some extremal dependence at time lag 1 for O3 in both  summer and winter.  This dependence mostly dissipates after one day.  The situation for NO3 is a bit more complex. One still discerns some extremal dependence at a one day lag as indicated by the high-temperature in the heat maps along the diagonal and subdiagonal.  However, some clusters have similar shading for its center of mass, e.g., clusters 1 and 4 for winter, which suggests poor delineation between the clusters.  In addition, there is a stronger day effect in the summer than winter for NO2 and the dependence does not necessarily die out after one day lag as in the O3 case.

\begin{figure}[h]%
\hfill
\subfigure[summer]{\includegraphics[scale=0.45]{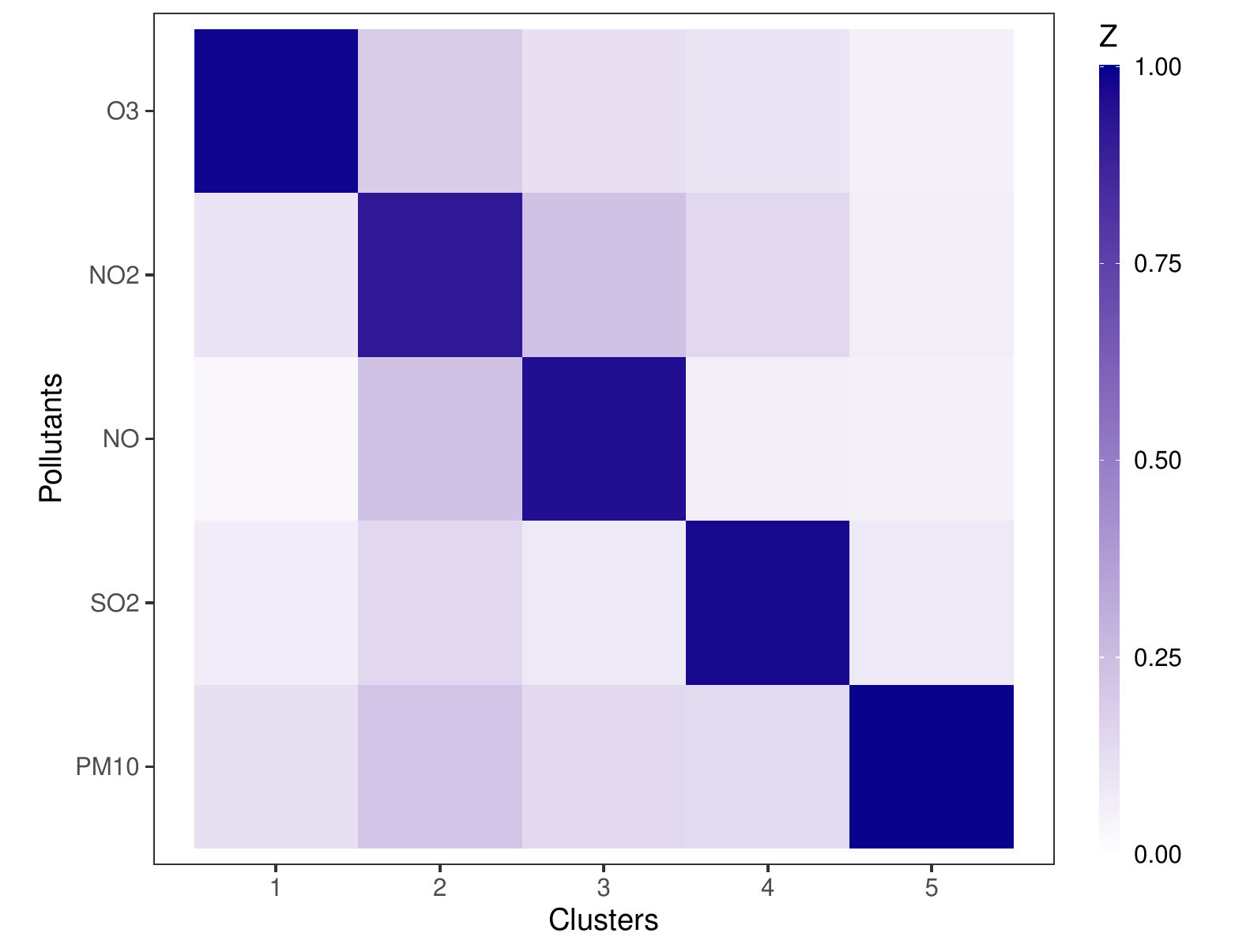}}
\hfil8
\subfigure[winter]{\includegraphics[scale=0.45]{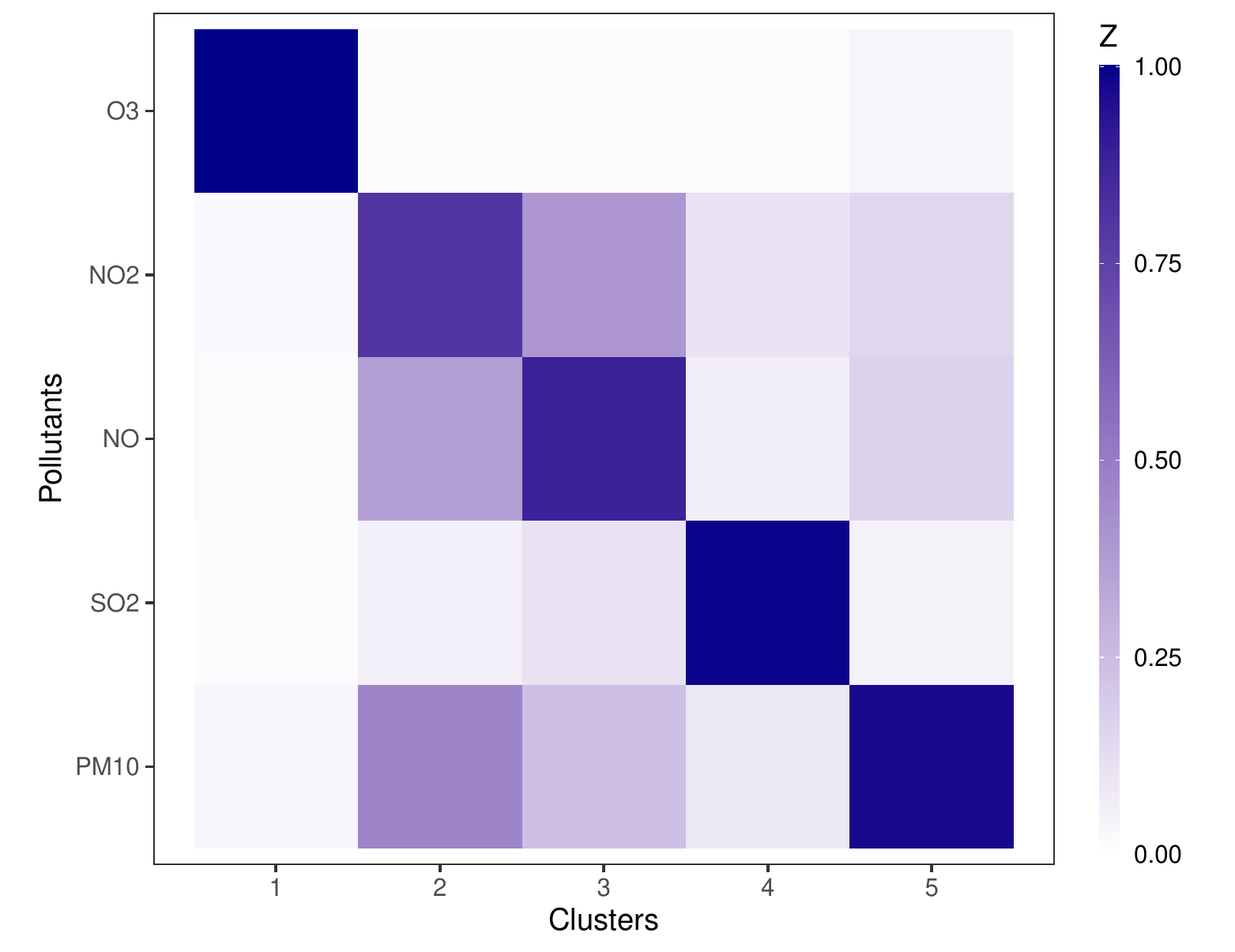}}
\caption{Five dimensional extremes from air pollution summer and winter data. The heat maps show the estimated cluster centers using spectral clustering with 5 clusters and 9-nearest neighbors.} 
    \label{fig:5clusters_pollution}
\end{figure}

\begin{figure}[h!]%
\hfill
\subfigure[summer]{\includegraphics[scale=0.42]{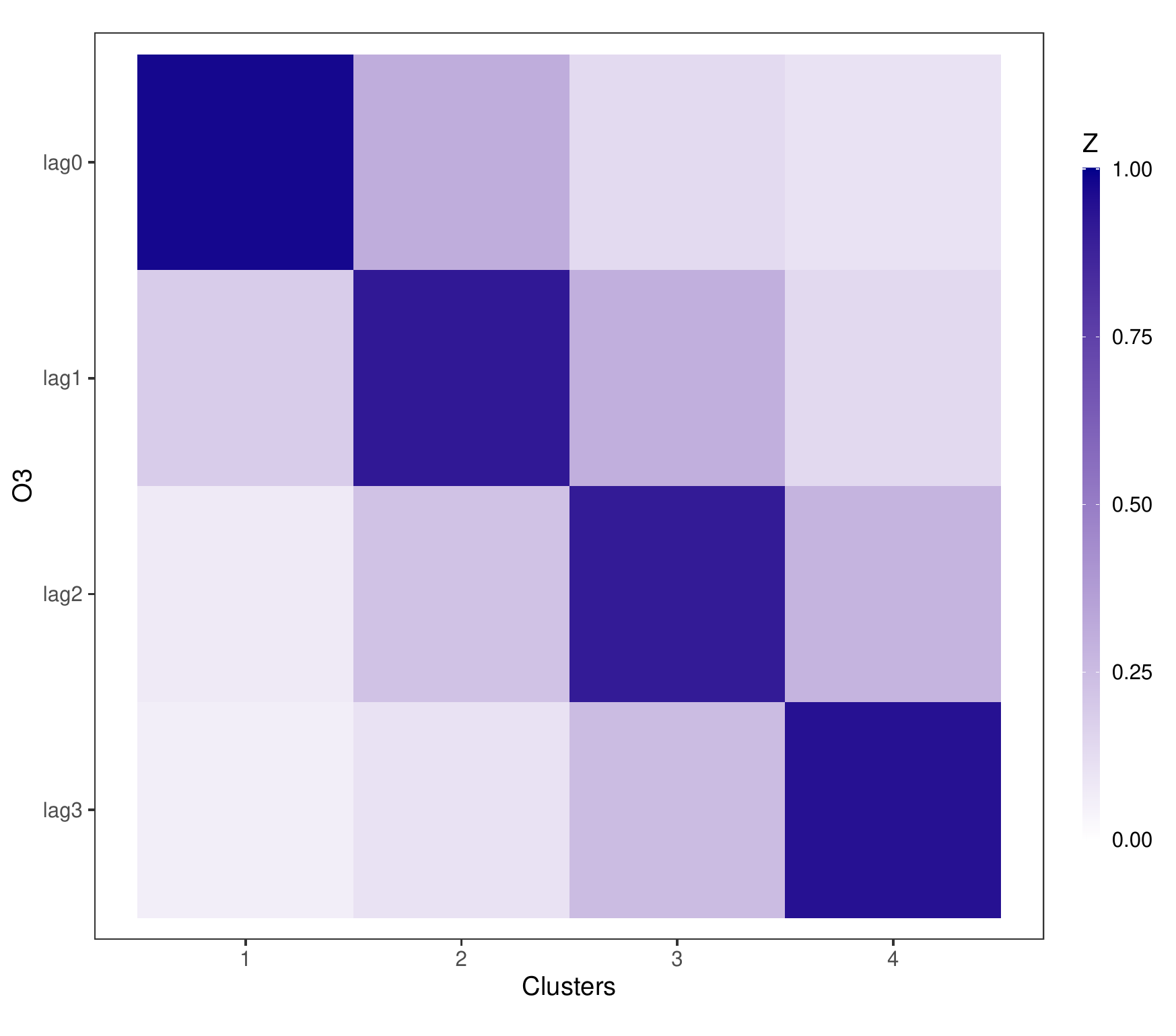}}
\subfigure[winter]{\includegraphics[scale=0.42]{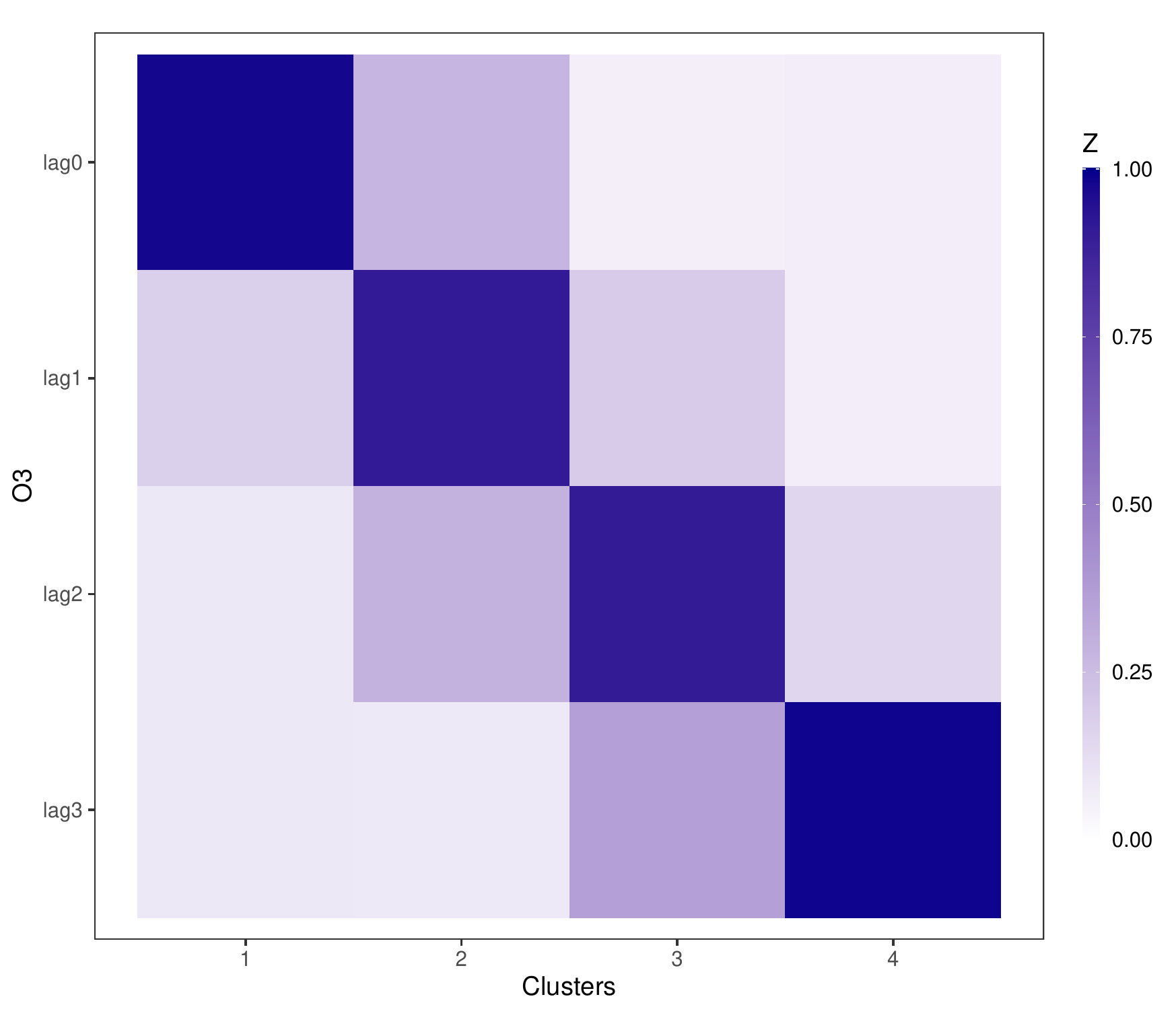}}
\caption{Four dimensional time series data constructed with  lags 0-3 of O3 for summer and winter data respectively. The heat maps show the estimated cluster centers using spectral clustering with 5 clusters and 9-nearest neighbors.} 
    \label{fig:O3_pollution}
\end{figure}

\begin{figure}[h!]%
\hfill
\subfigure[summer]{\includegraphics[scale=0.42]{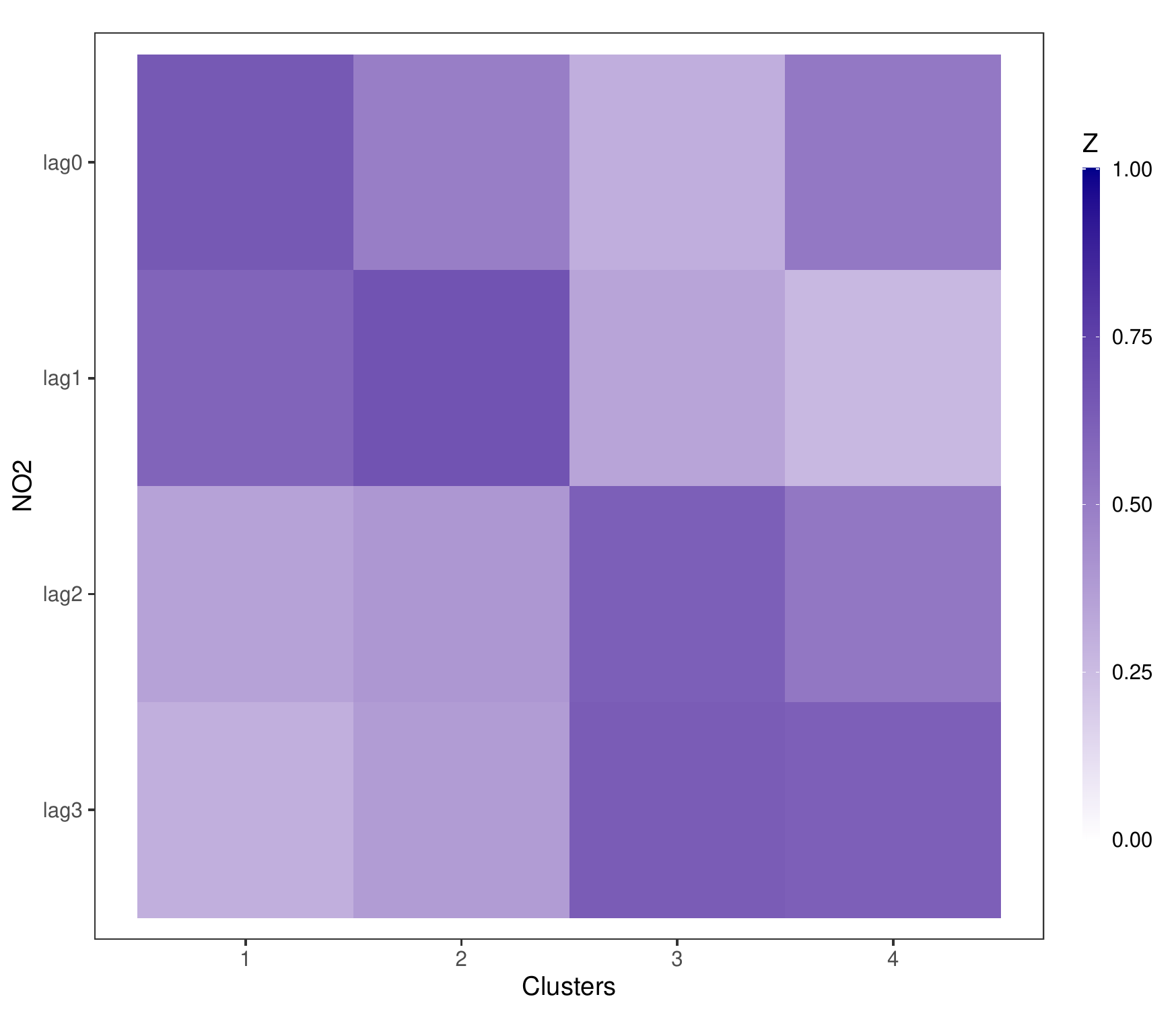}}
\subfigure[winter]{\includegraphics[scale=0.42]{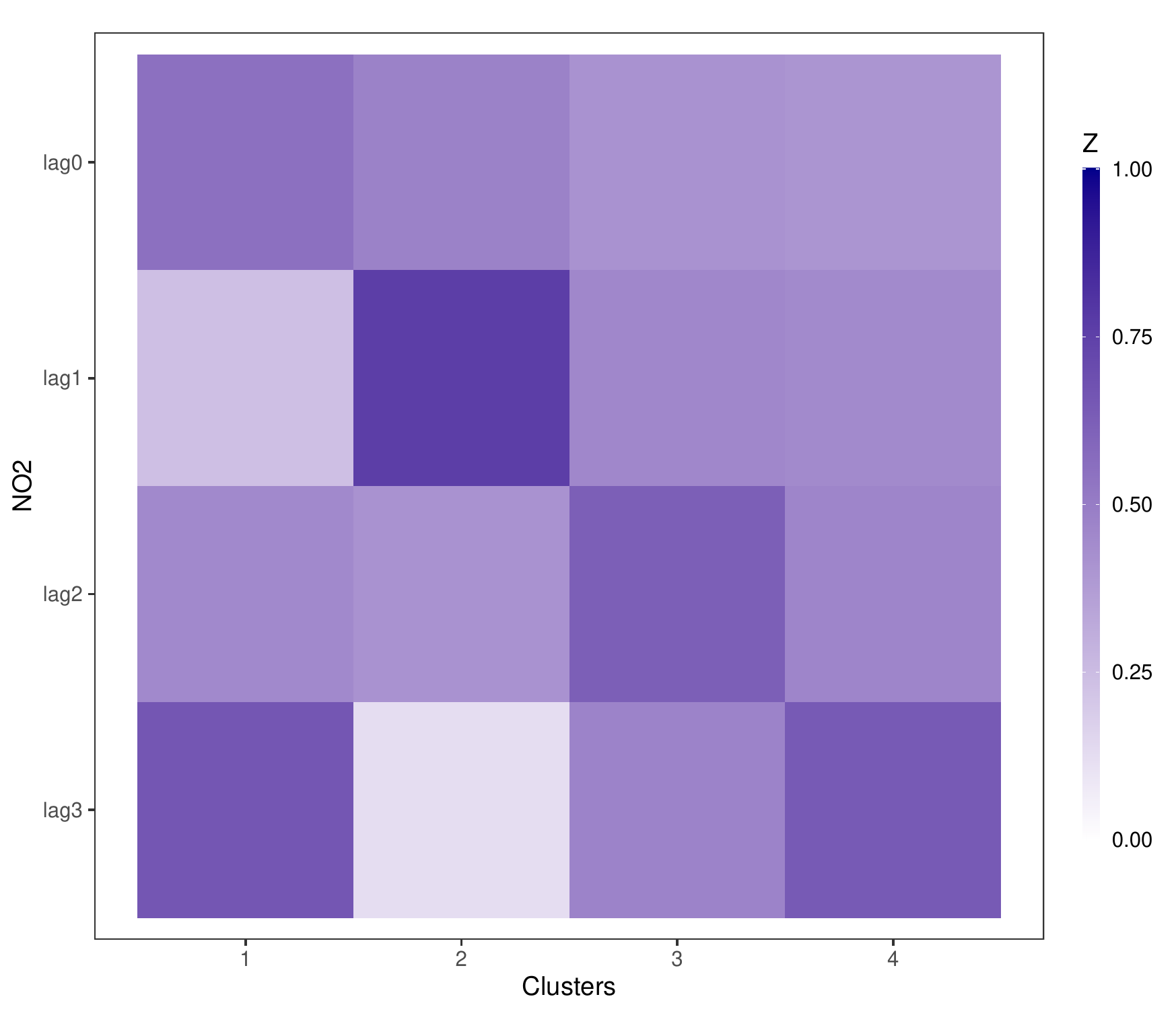}}
\caption{Four dimensional time series data constructed with  lags 0-3 of NO2 for summer and winter data respectively. The heat maps show the estimated cluster centers using spectral clustering with 5 clusters and 9-nearest neighbors. } 
    \label{fig:NO2_pollution}
\end{figure}

 \newpage

\section{Discussion}

In this work we introduced a spectral clustering approach for learning the angular measure of multivariate extremes. We proved that this approach leads to consistent clustering for a natural linear factor model and showed the good finite sample performance of our methods in numerical experiments. The encouraging results suggest the method might be applied in more general contexts. We are particularly interested in exploring two type of extensions. First, high dimensional scenarios where the dimension of the extremes $d$ might be larger than the number of observed extremes $N_n$. This would require introducing appropriate notions of sparsity and regularization. Second, it seems natural to investigate generative models that lead to continuous angular measures in the limit. This scenario implies one would need to carefully introduce more general definitions of extremal clusters and different analysis of the convergence of $k$-nearest neighbor graphs.



\section*{Appendix}

Before proving Lemmas \ref{l:densityT} and \ref{l:intervals} we will give a result regarding random partitions of uniform random variables that  we will leverage as the continuity of $F_{T_n}$ implies that  $F_{T_n}(T_n)\sim \mbox{Unif}(0,1)$. We remind the reader that in our proofs $c>0$ represents a finite and non-zero constant whose value may change from line-to-line.  


\begin{lemma}\label{l:partition}
Let $U_1,\dots,U_N\overset{iid}{\sim}\mbox{Unif}(0,1)$ and consider the random partition of the unit interval $I_{j,N}=[\frac{j-1}{m_N},\frac{j}{m_N})$, where $m_N=\frac{N}{\tau \log(N)}$, $\tau>1$ and $j=1,\dots,m_N$. Then, with probability at least $1-\frac{N^{1-\tau}}{\tau\log(N)}{ (1+N^{-0.2\tau})}$
\begin{enumerate}
\item[(i)] Every $I_{j,N}$ contains at least one of the variables $U_1,\dots,U_N$.
\item[(ii)] No $I_{j,N}$ contains more than $3\tau\log (N)$ of the variables $U_1,\dots,U_N$.
\end{enumerate}
\end{lemma}
\begin{proof}
Consider the event $E_{N,1}=\{ \mbox{Every } I_{j,N} \mbox{ contains at least one of the variables } U_1,\dots,U_N\}$ and note that a union bound gives
\begin{align}
    \label{e:probE1}
\nonumber    \mathbb{P}(E_{N,1})&\geq 1-\sum_{j=1}^{m_N}\mathbb{P}(U_k\notin I_{j,N}, \, \forall k=1,\dots,N)\\
\nonumber    &=1-m_N\left(1-\frac{1}{m_N}\right)^N \\
\nonumber & \geq 1-m_Ne^{-N/m_N} \\
& = 1-\frac{N^{1-\tau}}{\tau\log(N)}\,.
\end{align}
Now consider the event $E_{N,2}= \{ \mbox{No } I_{j,N} \mbox{ contains more than } 3\tau\log (N)\mbox{  of the variables }U_1,\dots,U_N \}$. It follows again from a union bound that yields
\begin{align*}
 \mathbb{P}(E_{N,2})&\geq 1-\sum_{j=1}^{m_N}\mathbb{P}( I_{j,N} \mbox{  has  more than } 3\tau\log(N) \mbox{ of the  }U_1,\dots,U_N)\\
    &=1-m_n\mathbb{P}(S_N>3\tau\log N),
\end{align*}
where $S_N\sim \mbox{Bin}(N,\frac{1}{m_N})$. Invoking Bernstein's inequality see that 
\begin{equation}
    \label{e:probE2}
     \mathbb{P}(E_{N,2})\geq  { 1-m_Ne^{-\frac{1}{2}\frac{(3\tau\log N-\tau\log N)^2}{\tau \log N +2\tau \log N/3}}= 1-\frac{N^{-(1.2\tau-1)}}{\tau\log N}}.
\end{equation}
Combining \eqref{e:probE1} and \eqref{e:probE2} shows that $(i)$ and $(ii)$ hold with the desired probability.
\end{proof}

\subsection*{Proof of Lemma \ref{l:densityT}}

{ We prove the lemma for positive $Z$. The same type of arguments work in the symmetric case and are therefore omitted.} Note that we can write for $y>0$, 
 {\begin{align}  \notag 
&f_{T_n}(t) = \frac{w_1^2}{c_pu_n}\cdot \\
\notag &\int_{-\infty}^\infty\cdots\int_{-\infty}^\infty\Bigl[
  z_1f_Z(z_1)\, f_Z(z_2)\cdots f_Z(z_{p-1})f_Z\bigl[\bigl(
  tz_1w_1^2/u_n-(c_2z_2+\cdots+c_{p-1}z_{p-1})\bigr)/c_p\bigr]   \\
\notag &\cdot\one\Bigl( z_1>u_n/w_1^{1/\alpha}, \, (a_{11}z_1+a_{12}z_2+\cdots+a_{1p}\tilde z_p)^2+(a_{21}z_1+a_{22}z_2\cdots +a_{2p}\tilde z_p)^2>u_n^2\Bigr)\Bigr]\\
 &\hskip 5.5in  dz_1 \cdots dz_{p-1}     \label{e:fT.expr}    \\
\notag  &\div\left[ \bbP\Bigl( (a_{11}Z_1+\cdots+a_{1p}Z_p)^2+(a_{21}Z_1+\cdots +a_{2p}Z_p)^2>u_n^2,
 \ \  {Z_1>u_n/w_1^{1/\alpha}}\Bigr)\right]\\
 & :=M_n(t)/D_n\,, \notag 
\end{align} 
where $\tilde z_p=(tz_1w_1^2/u_n-(c_2z_2+\cdots+c_{p-1}z_{p-1})\bigr)/c_p$, $c_i=a_{2i},\,i=1,\ldots,p$.}
We already know that 
\begin{equation}\label{e:Dn.asymp}
  D_n\sim cu_n^{-\alpha}, \ n\to\infty.
\end{equation}
Next, from \eqref{e:temp.ass},
$\sup_zf_Z(z)=M<\infty$, we conclude  by \eqref{e:density.bounds} that
\begin{align*} 
M_n(t) \leq  {\frac{Mw_1^{2}}{c_pu_n}}
\int_{u_n/w_1^{1/\alpha}}^\infty z_1f_Z(z_1)\, dz_1 \sim
  cu_n^{-\alpha}, \quad \mbox{ as } n\to\infty\,. 
\end{align*}
Hence there exists an
$ {G}\in (0,\infty)$ such that for all $n$ large enough, 
\begin{equation} \label{e:unif.bdd}
f_{T_n}(t) \leq  {G} \ \ \text{for all $t$.}
\end{equation}
This shows $(i)$. Let us now turn to claim $(ii)$  for concreteness consider $0< t\leq 1$. Note that, for large $n$,  the indicator in \eqref{e:fT.expr} is bounded from below by the indicator of the set $E=\{C^{-1}u_n<z_1<Cu_n, \, |z_i|\leq 1, \, i=2,\ldots, p-1\}$ for  some large $C$. Then, on $E$, the argument of the last function $f_Z$ in \eqref{e:fT.expr} is within a compact interval, so we obtain 
\begin{align*}  
M_n(t)\geq& c {\frac{Mw_1^{2}}{c_pu_n}}\int_{C^{-1}u_n}^{Cu_n}
  z_1f_Z(z_1)\, dz_1 \sim cu_n^{-\alpha}\,,  
\end{align*}
where the last relation follows from a direct application of \eqref{e:density.bounds}.  
Along with \eqref{e:Dn.asymp} this establishes $(ii)$.

Finally, note that
\begin{align*}
M_n(t)\leq& {\frac{w_1^2}{c_pu_n}}\int_{u_n/ {w_1^{1/\alpha}}}^\infty z_1f_Z(z_1)\int_\bbr \cdots \int_\bbr
 f_Z(z_2)\cdots f_Z(z_{p-1})\Bigl[\\
&~~f_Z\bigl[\bigl(
tz_1 {w_1^2}/u_n-(c_2z_2+\cdots+c_{p-1}z_{p-1})\bigr)/c_p\bigr]\Bigr]dz_1\cdots dz_{p-1}   \\
 =&\frac{ {w_1^2}u_n}{c_p}\int_{u_n/ {w_1^{1/\alpha}}}^\infty z_1f_Z(u_nz_1)\int_\bbr \cdots \int_\bbr
 f_Z(z_2)\cdots f_Z(z_{p-1})\Bigl[\\
 &~~f_Z\bigl[\bigl(
 tz_1 {w_1^2}-(c_2z_2+\cdots+c_{p-1}z_{p-1})\bigr)/c_p\bigr]\Bigr]dz_1\cdots dz_{p-1} \,.  
\notag
\end{align*}
Using the upper bound in \eqref{e:density.bounds} it is easy to see  that for some $c>0$  and sufficiently large $t$, 
\begin{align}
&\int_\bbr \cdots \int_\bbr 
f_Z(z_2)\cdots f_Z(z_{p-1})\Bigl[
f_Z\bigl[\bigl(
t-(c_2z_2+\cdots+c_{p-1}z_{p-1})\bigr)/c_p\bigr]\Bigr]dz_2\cdots dz_{p-1} \label{eq:integral}\\
&\leq ct^{-(\alpha+1)}\,.\notag
\end{align}
Indeed, the integral is, up to a constant, equal to the density of a linear combination of $Z_1,\ldots, Z_{p-1}$. 
Therefore, for all $y$ large enough, uniformly in $n$,
\begin{align*}
M_n(t) \leq& cu_n
\int_{1/ {w_1^{1/\alpha}}}^\infty z_1f_Z(u_nz_1)
(tz_1)^{-(\alpha+1)}\, dz_1 
\le cu_n^{-\alpha}t^{-(\alpha+1)}\,,
\end{align*}
where  once again we have used the upper bound in
\eqref{e:density.bounds}. Together with \eqref{e:Dn.asymp} this shows
the upper bound in $(iii)$. The lower bound in
$(iii)$ can be established in an identical way using  the lower bound in
\eqref{e:density.bounds}. $\Box$

\subsection*{Proof of Lemma \ref{l:intervals}}

It follows from Lemma \ref{l:partition} that, outside of an event
$\Omega_n^{(1)}$ with $\bbP\bigl( \Omega_n^{(1)}\bigr)\to 0$, 
each one of the intervals $I_{i,n}$ contains at least
one of the points 
$$T_{ni}=\frac{a_{21}Z_{2,i}^{(*,1)}+\dots+a_{p1}Z_{p,i}^{(*,1)}}{w_1^2Z_{1,i}^{(*,1)}/u_n},\quad i=1,\ldots,   N_n^{(1)},$$
and none of the intervals contains more
than $3\tau \log  N_n^{(1)}$ of these points. Note that 
\eqref{e:unif.bdd}
implies that 
$$\frac{\partial}{\partial t}F_{T_n}^{-1}(t)=\frac{1}{f_{T_n}(t)}\geq \frac{1}{ {G}},\quad \forall t \in\mathbb{R}$$
and hence by the fundamental theorem of calculus
$$ F_{T_n}^{-1}\left(\frac{i}{m_n}\right)-F_{T_n}^{-1}\left(\frac{i-1}{m_n}\right)\geq \frac{1}{ {G}m_n} $$
This shows that the length of the intervals $I_{i,n}$ 
satisfies 
\begin{equation} \label{e:not.short}
|I_{i,n}|\geq l_n/ {G}, \ \ i=1,\ldots, m_n,
\end{equation} 
 {where $l_n=\frac{1}{m_n}$.}
Since the   conditional law of $( Z_1/u_n,Z_2,\ldots,Z_p)$ given
\eqref{e:conditionT} converges
weakly, as $n\to\infty$, to the law of 
$$
\bigl( W_\alpha/w_1,Z_2,\ldots,Z_p\bigr)
$$
as defined in Theorem \ref{l:weak.limit}, we see that
\begin{equation*} \label{e:Ft.conv}
F_{T_n}\Rightarrow G:=\text{the law of} \ \
\frac{c_2Z_2+\cdots +c_pZ_p}{W_\alpha}.
\end{equation*}
It follows that the values $F_{T_n}(t_0)$ converge,
as $n\to\infty$, to a finite limit. Therefore, there is $0<\delta<1$
such that $F_{T_n}^{-1}\bigl( (i-1)/m_n\bigr)\geq t_0$ for all $n$
large enough and all $i\geq (1-\delta)m_n$. We conclude by
Lemma \ref{l:densityT} $(iii)$ that for such $n$ and $i$,
\begin{align} \label{e:bound.ln}
  l_n =& F_{T_n}\Bigl( F_{T_n}^{-1}\bigl( i/m_n\bigr)\Bigr) -
  F_{T_n}\Bigl( F_{T_n}^{-1}\bigl( (i-1)/m_n\bigr)\Bigr) \\
  \in& \bigl( D^{-1}, D\bigr) \int_{ F_{T_n}^{-1}\bigl(
       (i-1)/m_n\bigr)}^{F_{T_n}^{-1}\bigl( i/m_n\bigr)}
t^{-(\alpha+1)}\, dt.     \notag 
\end{align}
Furthermore,
\begin{equation}
    \label{e:bound.ln_aux1}
\int_{ F_{T_n}^{-1}\bigl(
       (i-1)/m_n\bigr)}^{F_{T_n}^{-1}\bigl( i/m_n\bigr)}
t^{-(\alpha+1)}\, dt\geq \Bigl( F_{T_n}^{-1}\bigl(
i/m_n\bigr)\Bigr)^{-(\alpha+1)} \Bigl( F_{T_n}^{-1}\bigl( i/m_n\bigr)
- F_{T_n}^{-1}\bigl(
(i-1)/m_n\bigr)\Bigr),
\end{equation}
while leveraging again Lemma \ref{l:densityT} $(iii)$ we see that
\begin{equation}
        \label{e:bound.ln_aux2}
 \frac{m_n-i}{m_n}= \int_{F_{T_n}^{-1}\bigl( i/m_n\bigr)}^\infty f_{T_n}(t)\, dt
  \leq D \int_{F_{T_n}^{-1}\bigl( i/m_n\bigr)}^\infty
  t^{-(\alpha+1)}\, dt
  = \frac{D}{\alpha} \Bigl( F_{T_n}^{-1}\bigl(
  i/m_n\bigr)\Bigr)^{-\alpha}. 
\end{equation}
Combining  \eqref{e:bound.ln_aux1} and \eqref{e:bound.ln_aux2},  we conclude that
$$
\int_{ F_{T_n}^{-1}\bigl(
       (i-1)/m_n\bigr)}^{F_{T_n}^{-1}\bigl( i/m_n\bigr)}
t^{-(\alpha+1)}\, dt \geq
c\left(\frac{m_n-i}{m_n}\right)^{(\alpha+1)/\alpha}
\Bigl( F_{T_n}^{-1}\bigl( i/m_n\bigr)
- F_{T_n}^{-1}\bigl(
(i-1)/m_n\bigr)\Bigr), 
$$
and so by \eqref{e:bound.ln},
$$
l_n\geq c\left(\frac{m_n-i}{m_n}\right)^{(\alpha+1)/\alpha}|I_{i,n}|.
$$
Since an upper bound can be obtained in the same way, we conclude that
for some $D_1\geq 1$, for all $n$ large enough and all
$i\geq (1-\delta)m_n$,  
\begin{equation} \label{e:Iin.asymp}
  D_1^{-1}{  l_n} \left(\frac{m_n-i}{m_n}\right)^{-(\alpha+1)/\alpha}\leq
  |I_{i,n}|\leq D_1{  l_n}
  \left(\frac{m_n-i}{m_n}\right)^{-(\alpha+1)/\alpha}.
\end{equation}
Choose $K>2D_1^2$, and choose $K_0$ so that
\begin{equation}
    \label{eq:cond_K}
1+\frac{K}{K_0}<\left( \frac{K}{2D_1^2}\right)^{\alpha/(\alpha+1)}.
\end{equation}
Consider an interval $I_{i,n}$ with $(1-\delta)m_n\leq i\leq
m_n-K_0$. It follows from \eqref{e:Iin.asymp} and the choice of  {$i$} that
any point in $I_{i,n}$ is closer to any point in $I_{i+1,n}$ than to
any point in an interval $I_{j,n}$ with $j<i-K$. To see this  {it} suffices  to show that 
\begin{equation}
    \label{eq:neighbor_int}
    |I_{i,n}|+ {|I_{i+1,n}| < \sum_{k=1}^K|I_{i-k,n}|} 
\end{equation}
Using \eqref{e:Iin.asymp} and $i\leq m_n-K_0$, the left hand side of \eqref{eq:neighbor_int} can be upper bounded  {by}
\begin{align}
    \label{eq:neighbor_int1}
  \nonumber |I_{i,n}|+ {|I_{i+1,n}|} &\leq   D_1{  l_n}
  \left(\frac{K_0}{m_n}\right)^{-(\alpha+1)/\alpha}+D_1{  l_n}
  \left(\frac{K_0+1}{m_n}\right)^{-(\alpha+1)/\alpha}\\
  & < 2D_1{  l_n}
  \left(\frac{K_0}{m_n}\right)^{-(\alpha+1)/\alpha}.
\end{align}
 {For $i\geq (1-\delta)m_n$ and using \eqref{e:Iin.asymp} once again}, the right hand side of \eqref{eq:neighbor_int} can be lower bounded  {by}
\begin{align}
    \label{eq:neighbor_int2}
  \nonumber \sum_{k=1}^K {|I_{i-k,n}|} & \geq D_1^{-1}{  l_n}\sum_{k=1}^K\left(\frac{m_n-(i-k)}{m_n}\right)^{-(\alpha+1)/\alpha}  \\
  \nonumber & \geq D_1^{-1}{  l_n}K\left(\frac{m_n-i+K}{m_n}\right)^{-(\alpha+1)/\alpha}\\ 
  \nonumber &  \geq D_1^{-1}{  l_n}K\left(\frac{\delta m_n+K}{m_n}\right)^{-(\alpha+1)/\alpha} \\
   &  > D_1^{-1}{  l_n}K\left(1+\frac{K}{K_0}\right)^{-(\alpha+1)/\alpha}
\end{align}
It follows from \eqref{eq:neighbor_int1} and \eqref{eq:neighbor_int2} that a suffcient condition for establishing \eqref{eq:neighbor_int} is
\begin{equation*}
   2D_1
  \left(\frac{K_0}{m_n}\right)^{-(\alpha+1)/\alpha}  \leq D_1^{-1}K\left(1+\frac{K}{K_0}\right)^{-(\alpha+1)/\alpha}.
\end{equation*}
The last condition {  implies \eqref{eq:cond_K} for large $n$}. 
We conclude that, on
the event $\Omega_n^{(1)}$, in a $k_n$-NN graph with
\begin{equation} \label{e:kn.b1}
  k_n> 3(K+1)\tau \log N_n^{(1)}, 
\end{equation}
then all points $\bigl(V^{(j)}, \,
j=1,\ldots,   N_n^{(1)}\bigr)$ within  $I_{i,n}$ in the range
$(1-\delta)m_n\leq i\leq m_n-K_0$ 
will be connected
both to each other and to  such a point in each $I_{i-1,n}$ and
$I_{i+1,n}$. 
The next observation to make is that, as long as $\delta$ is small
enough, the sequence $ \bigl(F_{T_n}^{-1}( 1-\delta)\bigr)$ is bounded
from above. Therefore, by Lemma \ref{l:densityT} $(ii)$, uniformly in large
enough $n$,  the
density $f_{T_n}$ is bounded from below by, say, $a>0$ on the interval
$\bigl(0, F_{T_n}^{-1}( 1-\delta)\bigr)$. Therefore, for all large
enough $n$, 
\begin{equation} \label{e:not.long}
|I_{i,n}|\,\leq \,l_n/a, \ \  1\leq i\leq (1-\delta)m_n.
\end{equation} 
To see this, note that
$$\frac{\partial}{\partial t}F_{T_n}^{-1}(t)=\frac{1}{f_{T_n}(t)}\leq \frac{1}{a},\quad \forall t \in\mathbb{R}$$
and hence by the fundamental theorem of calculus
$$ |I_{i,n}|=F_{T_n}^{-1}\left(\frac{i}{m_n}\right)-F_{T_n}^{-1}\left(\frac{i-1}{m_n}\right)\leq \frac{1}{am_n}=\frac{l_n}{a} $$

It follows from \eqref{e:not.short} and  \eqref{e:not.long} that if
$K>\frac{2 {G}}{a}$ then any point in $I_{i,n}$ is closer to any point in
$I_{i-1,n}$ and in $I_{i+1,n}$ than to
any point in an interval $I_{j,n}$ with $j<i-K$ or with
$j>i+K$. Therefore,  on
the event $\Omega_n^{(1)}$, in a $k_n$-NN graph satisfying
\eqref{e:kn.b1},  all points $\bigl(T_{nj}, \,
j=1,\ldots,   N_n^{(1)}\bigr)$ within  $I_{i,n}$ in the range
$1\leq i \leq (1-\delta)m_n$ 
will be connected both to each other and to  such a point in each
$I_{i-1,n}$ and $I_{i+1,n}$. 
Indeed, to show this it suffices  {to} show again that \eqref{eq:neighbor_int} holds true in the range $1\leq i\leq (1-\delta)m_n$. It is easy to see that \eqref{e:not.short}, \eqref{e:not.long} and $K>\frac{2 {G}}{a}$ entail
\begin{equation*}
    |I_{i,n}|+ {|I_{i+1,n}|}\leq \frac{2l_n}{a} < \frac{K l_n}{ {G}}\leq \sum_{k=1}^K {|I_{i-k,n}|.} 
\end{equation*}

Finally, it is obvious that if $K>K_0$,
then on the same event   $\Omega_n^{(1)}$, in a $k_n$-NN graph satisfying
\eqref{e:kn.b1},  all points $\bigl(T_{nj}, \,
j=1,\ldots,   N_n^{(1)}\bigr)$ within  $I_{i,n}$ in the range
$m_n-K_0<i\leq m_n$ 
will be connected both to each other and to  a such a point in 
each $I_{i-1,n}$ and $I_{i+1,n}$.  

Summarizing the above discussion we conclude that  on
the event $\Omega_n^{(1)}$, in a $k_n$-NN graph satisfying
\eqref{e:kn.b1} with $K$ large enough, all points $\bigl(T_{nj}, \,
j=1,\ldots,   N_n^{(1)}\bigr)$ within  $I_{i,n}$ in the entire range
$1\leq i\leq m_n$ 
will be connected both to each other and to  a such a point in 
each $I_{i-1,n}$ and $I_{i+1,n}$.  In particular, the $k_n$-NN graph
will be connected.

We now translate this discussion to the random vectors $ \BM^{(i)}, \,
i=1,\ldots,   N_n^{(1)}$. We define intervals
along  vector $\boldb$ by
\begin{equation*} \label{e:J.in}
  J_{i,n}=I_{i,n}\boldb, \ =1,\ldots,   N_n^{(1)}\,.
\end{equation*}
Then, outside of the event
$\Omega_n^{(1)}$,  each one of these intervals contains at least
one of the points $\bigl(\BM^{(i)}, \,
i=1,\ldots,   N_n^{(1)}\bigr)$ and none of the intervals contains more
than $3\tau \log  N_n^{(1)}$ of these points. By \eqref{e:not.short}
the lengths of these intervals satisfy for some $ {G}_1>0$, 
\begin{equation*} \label{e:J.not.short}
|J_{i,n}|\geq l_n/ {G}_{1}, \ \ i=1,\ldots, m_n\,.
\end{equation*}
We finally note that by \eqref{e:In.size}, with probability tending to one $N_n\sim C nu_n^{-\alpha}$,  and therefore $ {G}>0$ and $n$ large enough ensure that \eqref{e:kn.b1} holds provided $k_n >  {G}\log n$. This concludes the proof.  $\Box$

\subsection*{Proof of Theorem \ref{pr:d2.same}} Lemma \ref{l:intervals}  gives us the connectivity of the extremal
$k_n$-NN graph for $k_n$ satisfying
\eqref{e:kn.b1} with $K$ large enough. The next step is to understand by how much the points $\bigl(\bM^{(i)}, \,
i=1,\ldots,   N_n^{(1)}\bigr)$  are shifted by adding to them $\bigl(\BD^{(i)}, \,
i=1,\ldots,   N_n^{(1)}\bigr)$ in \eqref{e:main.term}. Denote $\Omega_n^{(2)}= {B}_n^c$ as defined in Lemma
\ref{l:only.one}. Then $\bbP\bigl( \Omega_n^{(2)}\bigr)\to 0$ as
$n\to\infty$ and it is elementary to check that outside of $\Omega_n^{(2)}$ we have 
$\|\BD^{(i)}\|\leq ch_n^2/u_n$ for all $i=1,\ldots,   N_n^{(1)}$.
Recall
that by the choice of $h_n$ we have
\begin{equation*} \label{e:small.shift}
  h_n^2/u_n = o(l_n) \ \ \text{as $n\to\infty$.}
\end{equation*}

If we define new sets by 
$$
\tilde J_{i,n} = \bigl\{ \BM^{(j)}+ \BD^{(j)}:\, \BM^{(j)}\in 
  J_{i,n} \bigr\}, \ i=1,\ldots, m_n,
  $$
 then it follows immediately that  for large $n$, outside of the event 
$\Omega_n^{(2)}$, the new sets have the property described by Lemma \ref{l:intervals}, perhaps with a larger $K_0$. We already know that this
means that for large $n$, outside of $\Omega_n^{(1)}\cup
\Omega_n^{(2)}$,  the extremal
$k_n$-NN graph with $k_n$ satisfying
\eqref{e:kn.b1} with $K$ large enough, is connected. $\Box$

\bibliographystyle{plainnat}
\bibliography{bibfile.bib}

\end{document}